\documentclass[12pt, onecolumn, draftclsnofoot, peerreview]{IEEEtran}
\ifCLASSOPTIONcompsoc
  \usepackage[nocompress]{cite}
\else
  \usepackage{cite}
\fi

\usepackage{float}  % in your preamble, if not already added
\usepackage{epsfig}
\usepackage{graphicx}
\usepackage{amsmath}
\usepackage{amssymb}
\usepackage{algorithm}
\usepackage[noend]{algpseudocode}
\usepackage{booktabs}
\usepackage{multirow}
\usepackage{amsfonts,amsthm,bm}
\usepackage{mathtools}
\usepackage{algorithm}
\usepackage[noend]{algpseudocode}
\usepackage{hyperref}

\newlength{\continueindent}
\usepackage{etoolbox}
\makeatletter
\newcommand*{\ALG@customparshape}{\parshape 2 \leftmargin \linewidth \dimexpr\ALG@tlm+\continueindent\relax \dimexpr\linewidth+\leftmargin-\ALG@tlm-\continueindent\relax}
\apptocmd{\ALG@beginblock}{\ALG@customparshape}{}{\errmessage{failed to patch}}
\makeatother

\usepackage{etoolbox}
\usepackage{xspace}
\usepackage{bm}
\usepackage{amsthm}%
\usepackage{subfigure}%

\usepackage{amsmath,amsfonts,mathtools}

\newtheorem{prop}{Proposition}

\newcommand{\MyMapTemplatePrefixc}[4]{\expandafter#1\csname#3#4\endcsname{#2{#4}}} % it remembles a template: \#3#4 --> #2{#4}
\forcsvlist{\MyMapTemplatePrefixc {\def} {\mathcal}{c}} {A,B,C,D,E,F,G,H,I,J,K,L,M,N,O,P,Q,R,S,T,U,V,W,X,Y,Z}  % e.g., \cA --> \mathcal{A}

\newcommand{\MyMapTemplatePrefixtb}[5]{\expandafter#1\csname#4#5\endcsname{#2{#3{#5}}}} % it remembles a template: \#3#4 --> #2{#4}
\forcsvlist{\MyMapTemplatePrefixtb {\def} {\tilde}{\mathbf}{t}} {A,B,C,D,E,F,G,H,I,J,K,L,M,N,O,P,Q,R,S,T,U,V,W,X,Y,Z}  % e.g., \tA --> \tilde{A}
\forcsvlist{\MyMapTemplatePrefixtb {\def} {\tilde}{\mathbf}{t}} {0,1,a,b,c,d,e,f,g,h,i,j,k,l,m,n,o,p,q,r,s,u,v,w,x,y,z}  % e.g., \ta --> \tilde{a}

\newcommand{\MyMapTemplateNoPrefix}[3]{\expandafter#1\csname#3\endcsname{#2{#3}}}
\forcsvlist{\MyMapTemplateNoPrefix {\def} {\mathbf} } {0,1,a,b,c,d,e, f, g, h, i, j, k, l, m, n, o, p, q, r, u, v, w, x, y, z} % e.g., \a --> \mathbf{a}
\forcsvlist{\MyMapTemplateNoPrefix {\def} {\mathbf} } {A,B,C,D,E,F,G,H,I,J,K,L,M,N,O,P,Q,R,S,T,U,V,W,X,Y,Z}  % e.g., \A --> \mathcal{A}

\def\bt{\bm{\theta}}

\def\Pr{{P}}

\def\ie{\emph{i.e.}\@\xspace}
\def\eg{\emph{e.g.}\@\xspace}

\usepackage{booktabs} % Nice tables
\usepackage[table,dvipsnames]{xcolor}
\definecolor{rowblue}{RGB}{220,230,240}

\usepackage{xcolor}

\usepackage{pgfplots}
\usepackage{tikz}
\usepackage{caption}

\begin{document}

\title{Self-Paced Collaborative and Adversarial Network for Unsupervised Domain Adaptation}
% make the title area

\author{Weichen Zhang, Dong Xu,~\IEEEmembership{Fellow,~IEEE}, Wanli Ouyang,~\IEEEmembership{Senior Member,~IEEE}, and Wen Li}

\maketitle

% The paper headers
% \markboth{IEEE TRANSACTIONS ON PATTERN ANALYSIS AND MACHINE INTELLIGENCE ,~Vol.XX, No.XX, XXXX XXXX}%
% {Shell \MakeLowercase{\textit{et al.}}: Bare Advanced Demo of IEEEtran.cls for IEEE Computer Society Journals}

% \IEEEtitleabstractindextext{%
\begin{abstract}
This paper proposes a new unsupervised domain adaptation approach called Collaborative and Adversarial Network (CAN), which uses the domain-collaborative and domain-adversarial learning strategy for training the neural network. 
The domain-collaborative learning strategy aims to learn domain specific feature representation to preserve the discriminability for the target domain, while the domain adversarial learning strategy aims to learn domain invariant feature representation to reduce the domain distribution mismatch between the source and target domains. We show that these two learning strategies can be uniformly formulated as domain classifier learning with positive or negative weights on the losses.  We then design a collaborative and adversarial training scheme, which automatically learns domain specific representations from lower blocks in CNNs through collaborative learning and domain invariant representations from higher blocks through adversarial learning. Moreover, to further enhance the discriminability in the target domain, we propose Self-Paced CAN (SPCAN), which progressively selects pseudo-labeled target samples for re-training the classifiers. We employ a self-paced learning strategy such that we can select pseudo-labeled target samples in an easy-to-hard fashion. Additionally, we build upon the popular two stream approach and extend our domain adaptation approach for more challenging video action recognition task, which additionally considers the cooperation between the RGB stream and the optical flow stream. We propose the Two-stream SPCAN (TS-SPCAN) method to select and reweigh the pseudo labeled target samples of one stream (RGB/Flow) based on the information from another stream (Flow/RGB) in a cooperative way. As a result, our TS-SPCAN model is able to exchange the information between the two streams. 
Comprehensive experiments on different benchmark datasets, Office-31, ImageCLEF-DA and VISDA-2017 for the object recognition task, and UCF101-10 and HMDB51-10 for the video action recognition task, show our newly proposed approaches achieve the state-of-the-art performance, which clearly demonstrates the effectiveness of our proposed approaches for unsupervised domain adaptation.

\end{abstract}

% Note that keywords are not normally used for peerreview papers.
\begin{IEEEkeywords}
domain adaptation, transfer learning, deep learning, adversarial learning, self-paced learning.
\end{IEEEkeywords}

% ~\IEEEmembership{Member,~IEEE,}
        % John~Doe,~\IEEEmembership{Fellow,~OSA,}
        % and~Jane~Doe,~\IEEEmembership{Life~Fellow,~IEEE}% <-this % stops a space

% \IEEEcompsocitemizethanks{\IEEEcompsocthanksitem Weichen Zhang, Dong Xu and Wanli Ouyang are with the School of Electrical and Information Engineering, The University of Sydney, NSW, Australia. Dong Xu is the corresponding author. \protect\\
% % note need leading \protect in front of \\ to get a newline within \thanks as
% % \\ is fragile and will error, could use \hfil\break instead.
% E-mail: weichen.zhang@sydney.edu.au, dong.xu@sydney.edu.au, wanli.ouyang@sydney.edu.au.\hfil\break
% \IEEEcompsocthanksitem Wen Li is with the Computer Vision Laboratory, ETH Z̈urich, Switzerlan. \hfil\break
% E-mail: liwen@vison.ee.ethz.ch
% }

% \IEEEpeerreviewmaketitle
\IEEEdisplaynontitleabstractindextext

\ifCLASSOPTIONcompsoc
\IEEEraisesectionheading{\section{Introduction}\label{sec:introduction}}
\else
\section{Introduction}
\label{sec:introduction}
\fi

% \IEEEPARstart{I}{n} many visual recognition tasks, the training data used to learn a model and the testing data on which the model is applied often have different distributions. 
\IEEEPARstart{I}{n} many visual recognition tasks, the training data and the testing data often have different distributions. 
In order to enhance the generalization capability of the models learnt from training data to the testing data, many domain adaptation technologies were proposed \cite{duan2012domainkernel, baktashmotlagh2013unsupervised, duan2012domain, fernando2013unsupervised, duan2012visual, li2014learning, li2018visual} for different visual tasks, such as object recognition, video event recognition, and semantic segmentation by explicitly reducing the data distribution mismatch between the training samples in the source domain and the testing samples in the target domain.

\begin{figure}[H]
\begin{center}
\includegraphics[width=\linewidth]{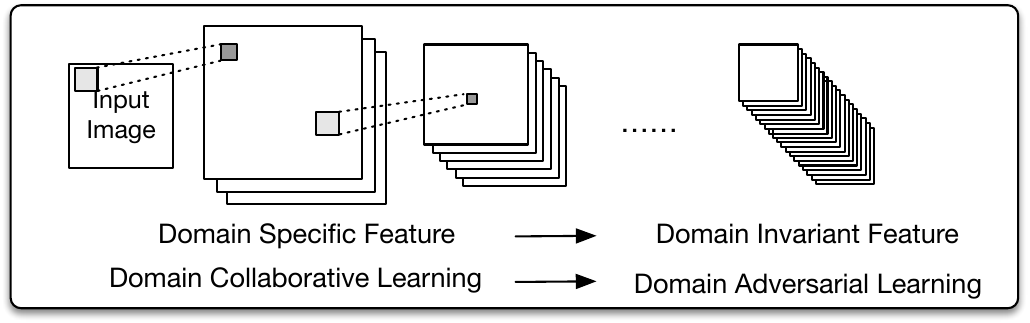}
\vspace{-7mm}
\end{center}
% \caption{Motivation of CAN. It is beneficial to learn domain informative features at lower layers such as corners and edges, which are useful for distinguishing not only images from different domains but also images from different classes, while domain uninformative features are useful for domain adaptation. }
\caption{Motivation of our CAN. We aim to learn domain specific features at lower layers and domain invariant features at higher layers, such that the learnt feature representation is domain invariant and the discriminability of features can also be well preserved.}
\label{fig:motivation_CAN}
\vspace{-4mm}
\end{figure}

The advancement of deep learning methods \cite{szegedy2015going, he2016deep, long2015fully, simonyan2014two, he2017mask}
% \wenli{please complete the reference. There are other missing reference below, please have a check all through the paper.}
have driven significant progress in a wide variety of computer vision tasks. However, modern deep learning methods are often based on the availability of large volumes of labeled data. To overcome this data limitation difficulty, deep transfer learning methods \cite{hu2015deep, long2015learning, long2016unsupervised, long2017jan, ganin2015unsupervised, ganin2016domain, tzeng2017adversarial, zeng2014deep, sun2016deep, bousmalis2016domain, liu2016coupled, kang2018deep, long2018conditional} are commonly used to learn domain invariant features by transferring the knowledge from the source domain to the target domain, in which the labeled samples are limited or even do not exist. Recently, a few deep transfer learning methods were proposed, which achieve promising results. These deep transfer learning approaches can be roughly categorized as statistical approaches \cite{long2015learning, long2016unsupervised, sun2016deep, long2018conditional}, which exploit regularizers, such as Maximum Mean Discrepancy (MMD) \cite{gretton2012kernel}, and adversarial learning based approaches \cite{bousmalis2016domain,ganin2015unsupervised,ganin2016domain,liu2016coupled,tzeng2017adversarial,kang2018deep}, which learn new representations through the adversarial learning processes. To learn domain-invariant representations for different domains, the recent work Domain Adversarial Training of Neural Network (DANN) \cite{ganin2016domain} added a domain classifier and used reversed gradient layer to update the feature extractor shared with the image/video classifier.  Please refer to Section 2 for a brief review of the existing domain adaptation approaches. More details are provided in our recent survey paper \cite{zhang2019recent}.
 
% For iCAN, we iteratively select a small amount of enlarged pseudo-labeled target samples that have high prediction confidence from the image classifier, and are also considered as domain uninformative based on the last domain classifier from the previous training epoch to better learn the target features with these more reliable pseudo-labels.

% We individually select different pseudo-labeled target samples for image classifier part and domain classifiers part respectively, to help each part better learn their own representation.

% \begin{figure}[H]
% \begin{center}
% \includegraphics[width=\linewidth]{TPAMI_SPCAN.png}
% \end{center}
%   \caption{Motivation of SPCAN. In iCAN, we select a small amount of high-confident samples to help learning the target representation. However, most of the samples are not selected. In SPCAN, we want to gradually select more lower-confident samples to affect the representation learning with giving them lower weight. So for SPCAN, inspired from Self-paced Learning, we iteratively select more pseudo-labeled target samples from easy to hard to gradually learn better target specific features for image classifier part, and better align the distributions of the features and categories across different domains for domain classifiers part. We individually select different pseudo-labeled target samples for image classifier part and domain classifiers part respectively, to help each part better learn their own representation. }
% \label{fig:motivation_SPCAN}
% \end{figure}

In this paper, we propose new unsupervised domain adaptation methods, in which all the source domain samples are fully annotated and none of the target domain sample is annotated. We apply our approaches in two common visual recognition tasks, object recognition in images and action recognition in videos. 

For the object recognition task, we propose a new deep transfer learning method called Collaborative and Adversarial Network (CAN), which integrates a set of domain classifiers (also called as domain discriminators) into multiple blocks. Each block consists of several CNN layers and each domain classifier is connected to one block.
As shown in Fig. \ref{fig:motivation_CAN}, our method CAN is based on the motivation
that some characteristic information from target domain
data may be lost after learning domain-invariant features
with the recent method DANN \cite{ganin2016domain}. 
Meanwhile, the representations at lower blocks (shallow layers) are often low-level features, such as corners and edges, which are expected to be specific for distinguishing images from different domains. As a result, we use collaborative and adversarial learning to learn discriminant features including domain-specific features in shallow layers and domain-invariant features in deep layers to reduce the data distribution mismatch. 
\begin{figure}[H]
\begin{center}
\includegraphics[width=\linewidth]{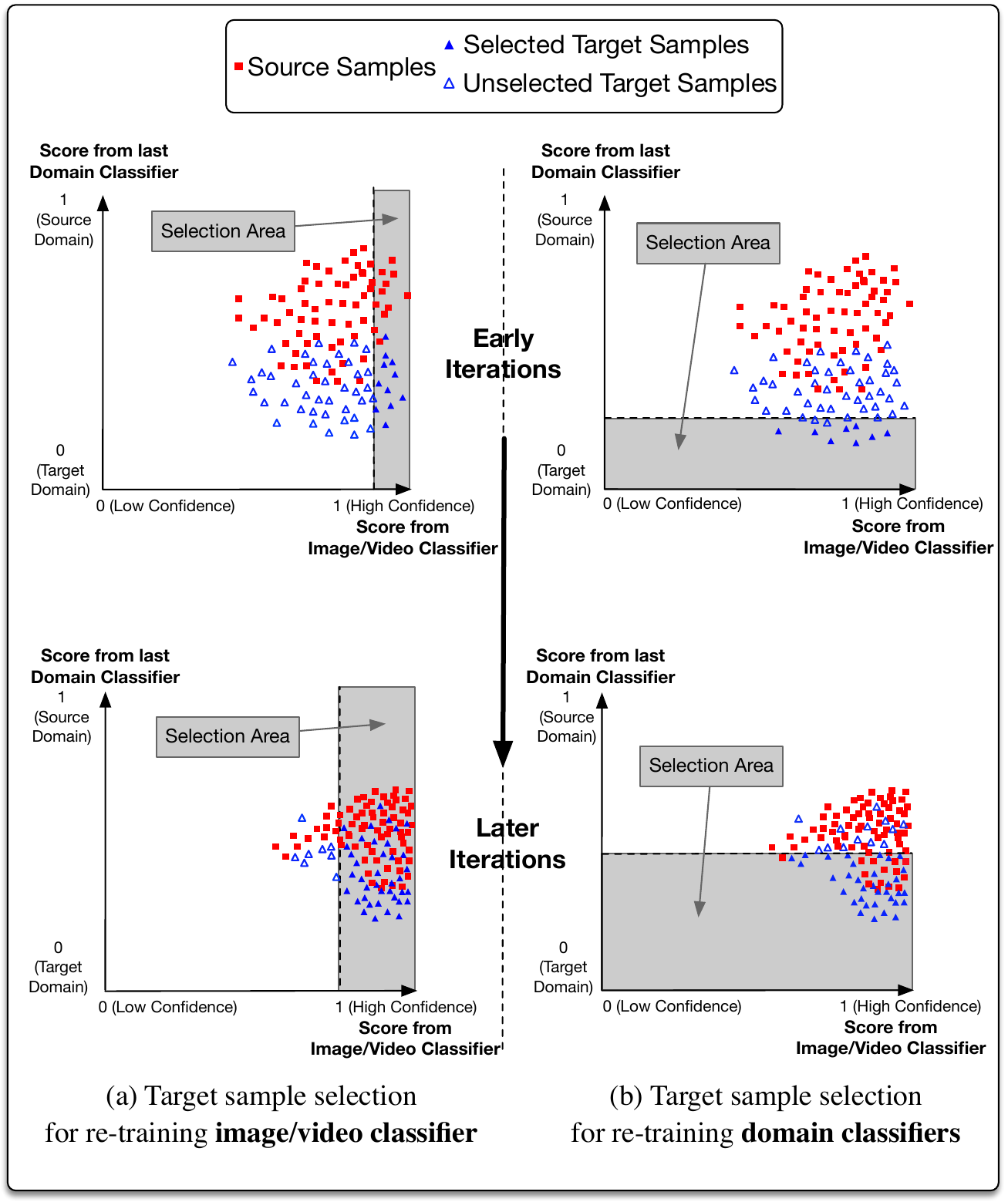}
\vspace{-6mm}
\end{center}
   \caption{Motivation of our Self-Paced CAN (SPCAN). In order to enhance the discriminability of the learnt model and inspired by Self-paced Learning, we re-train the CAN model by iteratively selecting pseudo-labeled target samples in an easy-to-hard learning scheme for both image/video classifier (a) and domain classifier (b). Based on the image/video classifier, our SPCAN gradually selects from fewer target samples to more target samples with high confidence scores. Based on the domain classifier, our SPCAN gradually selects from fewer target samples to more target samples that are likely from the target domain.}
\vspace{-2mm}
\label{fig:motivation}
\end{figure} 
Inspired by Self-Paced learning \cite{kumar2010self}, we additionally propose SPCAN based on an easy-to-hard progressive learning scheme. SPCAN selects pseudo-labeled target samples starting from few pseudo-labeled target samples with high confidence scores and gradually selects more pseudo-labeled target samples. SPCAN aims at, (1) gradually improving the image classification accuracy; and (2) gradually improving the domain classifier and reducing the data distribution mismatch between the source and  target domains. For the first goal (see Fig.  \ref{fig:motivation}(a)), the features and image/video classifier of SPCAN are additionally trained by adding a small amount of pseudo-labeled target sample with high confidence scores at the beginning, and gradually adding more samples in the subsequent iterations. For the second goal, as shown in Fig.  \ref{fig:motivation}(b),
% instead of the normal source and target sample classification process, 
the domain classifiers are gradually learnt by selecting from easy samples to harder ones, which can additionally improve the performance.  
%In addition, by reversing the gradients of the easily-classified samples, the features of these samples would become more domain-invariant and harder to be classified by the domain classifiers. The two processes of SPCAN are performed alternatively. 
% We also conduct careful component analysis to empirically compare the effectiveness of different combinations of selection criteria for both image and domain classifiers.

For the action recognition task, we additionally propose TS-SPCAN to learn better video-level feature representations under the unsupervised domain adaptation setting. We utilize the temporal segment network(TSN) \cite{wang2016temporal} as the base network, which fuses the RGB stream and the optical flow stream to predict the action classes. We propose a two-stream domain adaptation network, Two-stream SPCAN (TS-SPCAN), for both RGB and optical flow streams. Due to complementary representations in the two streams, in TS-SPCAN, the more reliable pseudo-labeled target samples of one stream (RGB/Flow) are used for another stream (Flow/RGB), which helps both streams to select more reliable pseudo-labeled target samples and substantially boost the domain adaptation results.

Comprehensive experiments on several benchmark datasets Office-31, ImageCLEF-DA and VISDA-2017 for the object recognition task and UCF101-10 and HMDB51-10 for the video action recognition task are conducted in Section \ref{sec:exp}. The results clearly demonstrate the effectiveness of our methods for both tasks. 
%The overall conclusion are drawn in Section 8.

A preliminary version of this paper was presented in  \cite{zhang2018collaborative}. This paper extends the work in \cite{zhang2018collaborative} by additionally proposing two new methods SPCAN and TS-SPCAN for domain adaptation as well as new experiments to evaluate the effectiveness of SPCAN and TS-SPCAN.
% \vspace{-2mm}
\section{Related Work}
\subsection{Domain Adaptation}
Classical domain adaptation methods can be roughly categorized as feature (transform) based approaches \cite{baktashmotlagh2013unsupervised,fernando2013unsupervised,kulis2011,gopalan2011domain,gong2012geodesic,sugiyama2008direct}, which aims to seek new domain-invariant features or learn new feature transforms for domain adaptation, and classifier based approaches \cite{duan2012domain, huang2007correcting, bruzzone2010domain, bickel2007discriminative,li2018domain}, which directly learn the target classifiers (\textit{e.g.}, the SVM based classifiers) for domain adaptation. 

Moreover, several deep transfer learning methods were proposed recently based on the convolutional neural networks (CNNs), which can be roughly categorized as statistic-based approaches \cite{mancini2018boosting, sun2016deep, carlucci2017autodial, long2017jan, long2015learning, long2016unsupervised,pinheiro2018unsupervised} and adversarial learning based approaches \cite{tzeng2015simultaneous, ganin2015unsupervised,ganin2016domain,bousmalis2016domain,liu2016coupled,tzeng2017adversarial, sankaranarayanan2017generate, hu2018duplex, russo2017source, hoffman2017cycada, lee2018diverse, saito2018maximum,rozantsev2018residual, rozantsev2018beyond}. The statistic-based approaches usually employ statistic-based metrics to model the domain difference. For example, Deep Correlation Alignment(CORAL) \cite{sun2016deep} utilized the mean and variance of the different domains, while Deep Adaptation Network \cite{long2015learning} adopt the MMD\cite{gretton2012kernel} as the domain distance. On the other hand, the adversarial learning based approaches used Generative Adversarial Networks (GANs) \cite{goodfellow2014generative} to learn domain-invariant representation, which can be explained as minimizing the $\cH$-divergence~\cite{ben2010theory} or the Jensen-Shannon divergence~\cite{gulrajani2017improved} between two domains.

Our work is more related to adversarial learning based approaches. The adversarial learning strategy has been widely used in domain adaptation for learning domain invariant representation. For example, the works in \cite{tzeng2015simultaneous, ganin2016domain, tzeng2017adversarial} introduced different domain loss terms to learn domain-invariant representation by adding additional classifiers or using the adversarial learning strategy. As a result, it is difficult to distinguish which domain each sample belongs to. Besides learning domain invariant representation, a few works have been proposed to preserve certain representation differences between different domains for better extracting features in the target domain. In particular, Domain Separation Network (DSN)  \cite{bousmalis2016domain} introduced one shared encoder and two private encoders for two domains by using the similarity loss and the difference loss.  As a result, the shared and private representation components are pushed apart, whereas the shared representation components are enforced to be similar. The works in \cite{rozantsev2018beyond,rozantsev2018residual} used the two-stream network for the source and target domains, in which two different regularization terms are introduced to learn the domain-invariant representation, respectively.  The approach in \cite{lee2018diverse} proposed an image-to-image translation framework to disentangle the domain-invariant features and the domain-specific features (attributes) and then use both types of features to translate the labelled images from the source domain to the target domain for training the classifier for domain adaptation. Different from these works, we propose a new collaborative and adversarial learning strategy to not only learn more domain-invariant representations in deeper layers through domain adversarial learning, but also learn domain-specific representations in shallower layers through domain collaborative learning. We used a shared model for both source and target domains and automatically learn the domain-specific and domain-invariant features by using different domain classifiers and learning the corresponding weights for different layers. In this way, the discriminability of the low-level features can be well preserved, while the common high-level semantics features can also be well extracted.

%Among adversarial learning based approaches, most works \cite{bousmalis2017unsupervised,liu2016coupled,sankaranarayanan2017generate,hu2018duplex, russo2017source, hoffman2017cycada,lee2018diverse} are based on Generative Adversarial Networks (GANs)\cite{goodfellow2014generative} by using generators to synthesize images or representations in different domains in order to learn domain invariant features. Moreover, Grifary \textit{et al.}~\cite{ghifary2016deep} proposed the Deep Reconstruction Classification Network (DRCN) method to learn a shared representation for classifying labeled source data and reconstructing unlabeled target data. Our work is more related to DANN \cite{ganin2016domain}, which learns domain-invariant features by inversely backpropagating the gradients from the domain classifier. In addition to learn domain uninformative representations at the last layer as in \cite{ganin2016domain}, we additionally learn domain informative representations in lower layers through domain collaborative learning which can improve visual recognition performance. 

%\subsection{Progressive Training Paradigm}
%In our extended method iCAN\cite{zhang2018collaborative} and SPCAN, progressive paradigm is used to select and reweigh the pseudo-labeled target samples in each iteration. 
\vspace{-1mm}
\subsection{Progressive Training for Domain Adaptation}
A few works have also proposed to progressively select pseudo-labeled target samples for improving unsupervised domain adaptation.  In \cite{bruzzone2010domain}, Bruzzone~\textit{et al.} proposed the Domain Adaptation Support Vector Machine (DASVM) method to iteratively select the unlabeled target domain data while simultaneously remove some labeled source samples. Chen~\textit{et al.}~\cite{chen2011co} proposed a features selection approach to split the target features into two views, and use co-training~\cite{blum1998combining} to leverage target domain pseudo-labeled samples in a semi-supervised learning fashion. More recently, Saito~\textit{et al.}~\cite{saito2017asymmetric}, equally leveraged three classifiers with an asymmetric tri-training network, where they selected pseudo-labeled target samples by the agreement of the two network classifiers, and then use these samples to retrain the third classifier to learn the target-discriminative features.
In our preliminary work, we propose a method called iCAN \cite{zhang2018collaborative} in which the domain classifier and the image/video classifier are both used to guide selecting more reliable samples. However, due to the large domain difference, the progressive training procedure may not work well when the labels of the selected pseudo-labeled samples are wrongly predicted. 

Different from those works, in this work, we propose a SPCAN model, which is inspired by the Curriculum Learning (CL)\cite{bengio2009curriculum} and Self-Paced Learning~\cite{kumar2010self}, to select target pseudo-labeled samples in an easy-to-hard fashion, thus improving the stability of progressive training. While several works were proposed to incorporate self-paced learning for object detection and semantic segmentation~\cite{tang2012shifting,zhang2017curriculum,zou2018unsupervised}, we aim to propose a new general unsupervised domain adaptation method for object and action recognition. Correspondingly, we design a new self-paced sample selection module.

% Our work SPCAN is inspired by the Curriculum Learning (CL)\cite{bengio2009curriculum} and Self-Paced Learning (SPL) \cite{kumar2010self}. Bengio \textit{et al.} proposed Curriculum Learning (CL)\cite{bengio2009curriculum}, which gradually learnt from easy samples to complex samples in a predefined scheme. Self-Paced Learning (SPL) \cite{kumar2010self} was then proposed, which adapted curriculum design as a regularization term to update model automatically. Many supervised or semi-supervised works applied this iterative training strategy for different visual recognition tasks \cite{zhang2017spftn, fan2017complex, shen2014unsupervised, wu2018exploit}. Tang et al. \cite{tang2012shifting} proposed a self-paced domain adaptation approach for object detection task from images to videos by learning with pseudo-labels in an easy-to-hard way. Zhang et al. \cite{zhang2017curriculum} proposed a curriculum-style domain adaptation approach for semantic segmentation task. In \cite{zou2018unsupervised}, Zou \textit{et al.} proposed a class-balanced self-learning method to deal with the semantic segmentation problem under unsupervised domain adaptation setting.
% In contrast to the existing methods that applied the progressive learning strategy for the supervised object classification task, or for the object detection or the semantic segmentation task, this work incorporates unsupervised domain adaptation framework with the self-paced learning strategy to progressively learn better representations for object and action recognition tasks.
\vspace{-1mm}
\subsection{Multi-View Domain Adaptation}

Recently, a few multi-view domain adaptation methods are proposed. In \cite{zhang2011multi}, zhang \textit{et al.} proposed a tranfer learning framework called Multi-View Transfer Learning with a Large Margin Approach (MVTL-LM) to integrate features of the source domain and the target domains from different views. The work in \cite{yang2013multi} incorporated multiple views of data in a perceptive transfer learning framework and proposed a domain adaptation approach called Multi-view Discriminant Transfer (MDT) to find the optimal discriminant weight vectors for each view by maximizing the correlation of the vector projections of two-view data and minimising the domain discrepancy and view disagreement simultaneously.
Niu \textit{et al.} proposed a multi-view domain generalization (MVDG) method \cite{niu2015multi}, which used exemplar SVM and alternating optimization algorithm to exploit
the manifold structure of unlabeled target domain samples for domain adaptation. The work in
\cite{ding2018robust} formulated a unified multi-view domain adaptation framework and conducted a comprehensive discussion across both multi-view and domain adaptation problems.  In contrast to these works, our newly proposed two-stream unsupervised domain adaptation approach TS-SPCAN focuses on progressively exchanging complementary information between the RGB and optical flow streams for the action recognition task under the unsupervised domain adaptation setting in an end-to-end fashion.
% In this way, the selected pseudo-labels of one stream(RGB/flow) can help another stream(flow/RGB), which boost the adaptation performance. 

\vspace{-1mm}
\subsection{Domain Adaptation for Action Recognition }
% Video-based action recognition task has been widely studied in these years[]. The main approaches can be categorized as learning the hand-crafted features[] and the deep learning features[] for recognizing the video-based actions. 

% The hand-crafted features are used to describe the spatial-temporal information of the video. Learning the hand-crafted representations is usually performing the spatial-temporal point extraction, feature quantization and feature encoding. Usually, a local 3D space is extracted from the Spatio-Temporal Interest Point(STIP) or the trajectories of the video. Then, the quantized local descriptors are then encoded into the action feature vector with fixed dimensionality, for instance, Histogram of Gradient and Optical Flow(HOG/HOF)\cite{laptev2008learning}, 3D histogram of Optical Flow(HOG3D)\cite{klaser2008spatio}, Motion Boundary Histogram(MBH)\cite{wang2013dense}. The existing encoding methods are then employed to combine the local descriptors to improve the action recognition performance, including Bag-of-Visual Words(BoVW)\cite{csurka2004visual}, Vector of Locally Aggregated Descriptors(VLAD)\cite{jegou2012aggregating}, Fish Vector(FV)\cite{sanchez2013image}, etc.
% \textcolor{blue}{(Needed?)}

Convolutional Neural Networks(CNNs) have been extensively exploited for the video-based action recognition task \cite{wang2013action, ji20133d, karpathy2014large, simonyan2014two, yue2015beyond, wang2015action, tran2015learning, feichtenhofer2016convolutional, wang2016temporal}. Recently, a deep learning based approach called Temporal Segment Networks(TSN) \cite{wang2016temporal} has been proposed to learn the video features based on the two stream framework (RGB and Optical flow), which used the sparse sampling strategy and fused several action prediction scores for each stream. 

The problem of domain shift between videos from different datasets is rarely explored. Recently, several transfer learning methods \cite{niu2016exploiting, yu2018exploiting, Jamal2018ddaa, busto2018open} have been applied for the action recognition tasks. Niu~~\textit{et al.}~\cite{niu2016exploiting} has proposed a multiple-domain adaptation method for video action recognition by leveraging a large number of web images from different sources. HiGAN~\cite{yu2018exploiting} used the two-level adversarial learning approach to find domain-invariant feature representation from source images and target action videos. Jamal~\textit{et al.}~\cite{Jamal2018ddaa} presented a domain adaptation method by combining adversarial learning for extracting action features under a 3D CNN network.
In \cite{busto2018open}, Busto~~\textit{et al.}~proposed a domain adaptation framework for the action recognition task, where the videos in the target domain contains instances of categories that are not present in the source domain. However, they mainly treat features of a video as a single representation, while our TS-SPCAN incorporates co-training into our Self-Paced CAN to take advantage of the complementary information of RGB and optical flow streams in an iterative fashion, which boosts the adaptation performance.

% \textcolor{blue}{
% }

\vspace{-2mm}
\section{Collaborative and Adversarial Network}
\label{sec:can}
\begin{figure*}
\begin{center}
% \fbox{\rule{0pt}{2in} \rule{.9\linewidth}{0pt}}
\includegraphics[width=\linewidth]{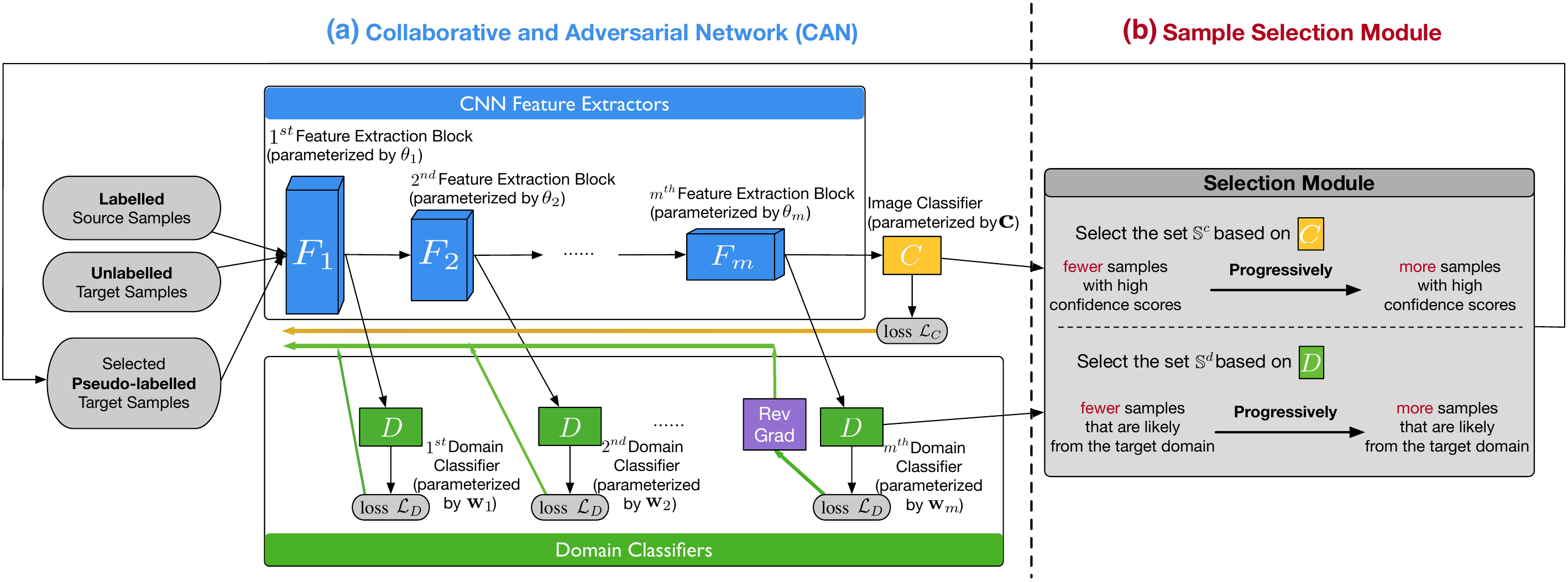}
\end{center}
\vspace{-12pt}
   \caption{The pipeline of our SPCAN method for image classification, which can also be similarly used for video classification. This pipeline consists of the CAN module (a) and the sample selection module (b). For the CAN module, we add multiple domain classifiers at different blocks of the CNN. We then train the domain classifiers by learning positive or negative weights on the losses to automatically extract domain-specific and domain-invariant features. As a result, the features learnt through domain collaborative learning at the lower layers will become more domain specific, while the features learnt through domain adversarial learning at the higher layers will become more domain invariant.  For our sample selection module, based on the image classifier, we gradually select the samples from fewer target samples to more target samples to construct the set $\mathbb{S}^c$; To construct set $\mathbb{S}^d$, based on the domain classifier, we gradually select the samples from fewer target samples to more target samples that are likely from the target domain. We use both selected sets $\mathbb{S}^c$ and $\mathbb{S}^d$ to retrain the CAN model in the next training epoch. As a result, the model gradually learns more discriminant features for image classification as well as aligning the distributions of two domains. }
\vspace{-4pt}
\label{fig:Overview}
\end{figure*}

\vspace{-1mm}
\subsection{Unsupervised domain adaptation in Object Recognition}
In unsupervised domain adaptation, we are given a set of labeled source samples and a set of unlabeled target samples, where the source and target data distributions are different. The goal is to learn a classifier, which can accurately categorize the unlabeled target samples into the given classes. In this section, we focus on the image classification task, and our approach can be readily extended for video action classification (see Section \ref{sec:cospcan}).

Formally, denote $\mathbf{X}^{s}=\{(\mathbf{x}^{s}_i,y^{s}_i)|_{i=1}^{N^{s}}\}$ as the set of samples from the source domain, where $\mathbf{x}^{s}_i$ is the $i$-th source domain sample, $y^s_i$ is its category label, and $N^s$ is the number of samples in the source domain. Similarly, we denote $\mathbf{X}^{t}=\{\mathbf{x}^{t}_{i}|_{i=1}^{N^{t}}\}$ as the set of samples from the target domain, where $\mathbf{x}^t_i$ is the $i$-th target domain sample, and $N^t$ is the number of samples in the target domain. For convenience, we also use $\mathbf{X} = \mathbf{X}^s \cup \mathbf{X}^t = \{(\mathbf{x}_i, d_i)|_{i=1}^{N}\}$ to denote the training samples from both domains, where $N$ is the total number of samples, $d_i \in \{0, 1\}$ is the domain label for the $i$-th sample with $d_i=1$ as the source domain, and $d_i = 0$ as the target domain.

There are many deep domain adaption approaches, which learn the domain invariant representations for the two domains \cite{long2015learning, ganin2016domain}. Usually, their methods are based on a multi-task learning scheme, which aim to minimize the classification loss on the labeled source data, and also align the source and target domain distributions with a domain discrepancy loss (\textit{e.g.}, $\cH$-divergence or MMD). 

However, since the network aims to minimize the source domain classification loss, domain specific features that are discriminative for image classification might be lost in the training process, making the learnt representation less discriminative for classifying images in the target domain. To this end, we propose a new Collaborative and Adversarial Network (CAN) model, which simultaneously learns domain-specific and domain-invariant features through domain collaborative and domain adversarial learning. Empirically, the final feature representation learnt by our CAN method is more discriminative for image classification. Our new method CAN with three parts is described step by step as follows.

\vspace{-2mm}
\subsection{Domain Specific Feature Learning}
Domain specific features are useful for distinguishing which domain each sample belongs to. In a general deep neural network, features learnt from lower layers are mostly low-level details, e.g. corners and edges, and features learnt at higher layers are semantic features that describe the objects or attributes. The low-level features are often domain specific that are useful for distinguishing not only images from different domains but also images from different classes. 

Therefore, different from the existing methods \cite{long2015learning, ganin2016domain} which learn domain invariant features for the whole network, our CAN method makes these low-level features distinguishable from different domains. To keep the domain specific features when training the whole network, we propose to introduce several domain classifiers, where each of the domain classifiers is applied on each block (including several CNN layers). 

In particular, given a training sample $\x$, let us denote its feature representation extracted from a certain layer (\textit{e.g.}, ``conv3'') as $\f$. We also denote the feature extraction network before this layer (inclusive) as $F$, then $\f$ can be deemed as the output of $F$ (\textit{ i.e.}, $\f = F(\x; \bt)$), where $\bt$ are the parameters for $F$. 

To let $\f$ encode as much target information as possible, we propose to learn a domain classifier $D: \f \rightarrow \{0, 1\}$, which is used to predict whether the input sample $\x$ belongs to the source domain (1) or the target domain (0). Intuitively, if the feature representation $\f$ can be used to well distinguish which domain the sample $\x$ comes from, sufficient target information should be encoded within $\f$. Then, the objective for learning the domain classifier $D$ can be written as,
\begin{eqnarray}
\label{eqn:discriminator}
\min_{\bt, \w}\frac{1}{N}\sum_{i=1}^N\cL_{D}\left(D(F(\x_i; \bt); \w), d_i\right),
\end{eqnarray}
where $\w$ is the parameters for the domain classifier $D$, $d_i$ is the domain label, and $\cL_{D}$ is the classification loss, where the binary cross entropy loss is used in this paper.

\vspace{-2mm}
\subsection{Domain Invariant Feature Learning}
Domain invariant features cannot clearly distinguish whether the sample is from the source or the target domain.
To learn the domain invariant representations at the final layer for image classification, other deep transfer learning approaches \cite{long2015learning, ganin2016domain} designed various domain alignment approaches to handle this problem. Our work uses domain adversarial learning, which learns a domain classifier and reversely back-propagate the domain gradient to confuse the domain classifier through a Gradient Reversal Layer \cite{ganin2015unsupervised}. In this way, the shared features learnt in the domain adversarial way is able to reduce the distribution mismatch between the source and target domain. With this domain adversarial learning strategy, the learnt image classifier performs better on the target domain. 

Specifically, by following the domain adversarial training strategy in DANN \cite{ganin2015unsupervised,ganin2016domain}, the objective for learning the domain invariant representations can be written as
% \vspace{-1mm}
\begin{eqnarray}
\label{eqn:dann}
\max_{\bt}\min_{\w}\frac{1}{N}\sum_{i=1}^N\cL_{D}\left(D(F(\x_i; \bt); \w), d_i\right),
% \vspace{-1mm}
\end{eqnarray}
where the difference between the above equation and Eqn. (\ref{eqn:discriminator}) is that the network $F$ is trained by minimizing the domain classification loss $\cL_{D}$ in Eqn. (\ref{eqn:discriminator}) for domain specific features, while $F$ is trained by maximizing $\cL_{D}$ in Eqn. (\ref{eqn:dann}) for domain invariant features. Thus, the network can be easily trained by minimizing the loss with the conventional optimization techniques like the stochastic gradient descent. 

Intuitively, we optimize  Eqn. (\ref{eqn:discriminator}) in order to distinguish the two domains, such that the feature representation generated from $F$ is domain specific. Instead, we optimize Eqn. (\ref{eqn:dann}) in order to remove domain specific information, such that two domains are similar to each other when using the feature representations generated from $F$. Thus, optimizing Eqn. (\ref{eqn:dann}) inevitably causes loss of target specific features, since such representation does not help to confuse the discriminator. To learn
domain specific representation at lower layers and domain invariant representation at higher layers, we propose a collaborative and adversarial learning scheme below, which accommodates the opposite tasks in Eqn. (\ref{eqn:discriminator})  and Eqn. (\ref{eqn:dann}) into one framework.

\vspace{-1mm}
\subsection{Collaborative and Adversarial Feature Learning Scheme}

%To accommodate the two opposite tasks, domain informative and domain uninformative feature learning should be performed together. We propose a
In the collaborative and adversarial learning scheme, we encourage the feature representation of the model to keep as much domain specific feature as possible at lower layers, while enforce the feature representation to be domain invariant at higher layers. The two tasks are applied to multiple layers with different weights, such that feature representation goes smoothly from domain specific to domain invariant when the samples are forwarded in the network from lower layers to higher layers. 

The architecture of our CAN model is illustrated in the left part of  Fig~\ref{fig:Overview}. We divide the whole CNN network into several groups, which are called blocks (\ie, the blue cubes). Each block consists of several consecutive CNN layers. Suppose in total $m$ blocks are used, we build a domain discriminator after the last layer of each block, leading to $m$ domain discriminators (\ie, the green rectangles). We assign a weight $\lambda_l$ ($l=1, \ldots, m$) for each domain discriminator, and automatically optimize the weights when we back-propagate the losses. 
Intuitively, a higher $\lambda_l$ indicates the network tends to learn domain specific features at the $l$\-th block. When $\lambda_l < 0$, the network tends to learn domain invariant features at the $l$\-th block.

For example, for a ResNet50 network, we divide the base model into four blocks, and add four discriminators after each block, where the sizes of the receptive fields are changed. Then, we add a domain discriminator and its corresponding weight to each block to learn domain specific and domain invariant features through our domain collaborative and adversarial learning method.

Formally, denote $\bt_l$ ($l=1, \ldots, m$) as the network parameters before the $l-$th block (inclusive), and denote $\w_l$ as the parameters of the domain discriminator at the $l-$th block. For $m$ blocks, we denote $\W = \{\w_1, \ldots, \w_m\}$, $\Theta_F = \{\bt_1, \ldots, \bt_m\}$ \footnote{We define $\Theta_F$ in a similar way as W for better presentation, despite the fact that in general $\Theta_F = \bt_m$ since $\bt_m$ usually includes the paramerters in $\bt_l$ for $l_m$}. We also denote the loss term from a domain discriminator as $\cL_{D}(\bt, \w) = \frac{1}{N}\sum_{i=1}^N\cL_{D}(D(F(\x_i;\bt);\w),d_i)$.

The objective of our collaborative and adversarial learning scheme can be written as follows:
% \vspace{-2mm}
% \begin{eqnarray}
% \!\!\!\!\min_{\Theta_F, \W,  \boldsymbol{\lambda}}\!\cL_{CA} \!\!\!\!&\!\!=\!\!\!&\!\!\!\sum_{l=1}^{m-1}
% \lambda_l\cL_D(\bt_l, \w_l)\!+\!\lambda_m\cL_D(\bt_m, \w_m),
% \label{eqn:ca}
% \\
% \mbox{s.t.} && \!\!\!\!\!\!\sum_{l=1}^{m-1} \lambda_l= \lambda_0,  \quad |\lambda_l| \leq \lambda_0, \nonumber
% \end{eqnarray}
\begin{eqnarray}
\!\!\!\!\min_{\Theta_F, \boldsymbol{\lambda}}\!\cL_{CA} \!\!\!\!&\!\!=\!\!\!&\!\!\!\sum_{l=1}^{m}
\lambda_l\min_{\W}\cL_D(\bt_l, \w_l),
\label{eqn:ca}
\\
\mbox{s.t.} && \!\!\!\!\!\!\sum_{l=1}^{m} \lambda_l= -1,  \quad |\lambda_l| \leq 1. \nonumber
\end{eqnarray}
where the set of the domain discriminator weight $\boldsymbol{\lambda} = \{\lambda_1, \ldots, \lambda_{m}\}$.  
% $\lambda_l$ is the weight for each block with $l=1, \ldots, m$, which is automatically optimized during back-propogation.(already written)
DANN \cite{ganin2016domain} only used one domain classifier and employed a gradient reversal layer to change the sign of the gradient back-propagated from the domain classifier (\textit{i.e.}, by multiplying by $-1$). In this work, we extend the DANN method to use multiple domain classifiers at different blocks, and automatically learn the weights for different domain classifiers. In order to be consistent with DANN \cite{ganin2016domain}, we set the sum of all weights $\lambda_l$'s to be $-1$.
By automatically optimizing the loss weights $\lambda_l$'s, we can reduce the number of hyper-parameters, and more importantly, we also allow multiple specific feature learning tasks to well collaborate with each other. When $\lambda_l \geq 0$,  the corresponding sub-problem is similar to the optimization problem in  Eqn. (\ref{eqn:discriminator}), so we disable the gradient reverse layer, and encourage the corresponding discriminator to learn domain specific features by performing the \textbf{Collaborative} learning process. On the other hand, when $\lambda_l < 0$, the corresponding sub-problem is similar to the max-min problem in Eqn. (\ref{eqn:dann}), so we enable the gradient reverse layer, and encourage the discriminator to learn domain invariant features by performing the \textbf{Adversarial} learning process. 

% Instead of using the negative weight, we use $|\lambda_l|$ as the loss weight and add a gradient reversal layer (GRL) before the domain discriminator to inversely back-propagate the domain gradient as in DANN

We can incorporate the loss $\cL_{CA}$ in (\ref{eqn:ca}) into any popular deep convolutional neural networks (CNNs) architecture (\textit{e.g.}, AlexNet, VGG, ResNet, DenseNet, \textit{etc}) to learn robust features for unsupervised domain adaptation. Therefore,  we jointly optimize the above loss with the conventional classification loss. Let an image classifier be $C: \f \rightarrow \tilde{y_i}$. The image classification loss can be denoted as
\begin{eqnarray}
\label{eqn:src}
\cL_{src} = \frac{1}{N^s}\sum_{i=1}^{N^s}\cL_{C}(C(F(\x^s_i; \Theta_F);\c), y^s_i),
\end{eqnarray}
where $\c$ is the parameters for the classifier $C$ and $\cL_{C}$ is the cross entropy loss for the classification task.

The final objective for our collaborative and adversarial network (CAN) can be written as
\begin{eqnarray}
\label{eqn:can}
\min_{\Theta_F, \c, \W, \lambda_l \in \Lambda}\cL_{CAN} = \alpha\cL_{CA} + \cL_{src},
\end{eqnarray}
where $\alpha$ is the trade-off parameter and $\Lambda = \{\lambda_l|\sum_{l=1}^{m} \lambda_l=-1, |\lambda_l| \leq 1, l=1,\ldots, m\}$ is the feasible set of $\lambda_l$'s. % While we set $\sum_l \lambda_l=-1$ (see the constraint in Eqn. (\ref{eqn:ca})), we additionally introduce the trade-off parameter $\lambda$ as in \cite{ganin2016domain} in order to provide more flexibility.

As discussed in \cite{ganin2016domain}, the domain adversarial training in Eqn. (\ref{eqn:dann}) can be seen as to minimize the $\cH$-divergence between two domains. Formally, given two domain distributions $\cD^s$ of $\X^s $and $\cD^t$ of $\X^t$, and a hypothesis class $\cH$(a set of binary classifiers $\textit{h}: \X \rightarrow [0,1]$), the $\cH$-divergence for the two distributions $d_\cH(\cD^s,\cD^t)$ is then defined as~\cite{ben2007analysis,ben2010theory},
\begin{eqnarray}
\label{eqn:dh}
d_\cH(\cD^s,\cD^t) = 2 \sup_{\textit{h}\in \cH}\Big|\Pr{_{\x\sim \cD^s}}[\textit{h}(\x)=0] - \Pr_{\x\sim\cD^t}[\textit{h}(\x)=0]\Big|.\nonumber
\end{eqnarray}
It has been shown in \cite{ben2007analysis,ben2010theory} that when $\cH$ is a symmetric hypothesis class, the $\cH$-divergence can be empirically computed by, 
% and the \textit{empirical $\cH$-divergence} between the two samples of $\X^s$ and $\X^t$ from two domains can be computed by,
\begin{eqnarray}
\begin{aligned}
\label{eqn:hat_dh}
\hat{d}_\cH(\X^s,\X^t) = 2 \Bigg(1- \min_{\textit{h}\in \cH}\bigg [\!\!\!\!\!\!\!\!\!\!\!\!\!\!\!\! &&\frac{1}{N^s}\sum_{i=1}^{N^s}\textit{I}\Big(\textit{h}(\x_i^s)=1\Big) \\
&&+\frac{1}{N^t}\sum_{i=1}^{N^t}\textit{I}\Big(\textit{h}(\x_i^t)=0\Big)\bigg ]\Bigg),
\end{aligned}
\end{eqnarray}
where $I(a)$ is the binary indicator function. It can be easily observed that it is equivalent to optimize the objective in Eqn. (\ref{eqn:dann}) or find a hypothesis $\h$ that minimizes $\hat{d}_{\cH}(\X^s,\X^t)$. 

The DANN \cite{ganin2015unsupervised, ganin2016domain} approach is designed to simultaneously minimize the classification error on the labeled samples in the source domain (\ie, Eqn. (\ref{eqn:src})), and the empirical $\cH$-divergence between two domains (\ie, Eqn. (\ref{eqn:dann})) when training neural networks. To achieve this, DANN adds a domain discriminator at the final layer, and reversely back-propagates the gradient learnt from the domain discriminator to make the feature extracted from the final layer become domain indistinguishable. 

Similarly, in our CAN, we introduce a domain discriminator at each layer, and minimize the domain discrepancies measured at all layers. Therefore, the DANN approach can be seen as a special case of our CAN when we only have one domain discriminator at the final layer. By jointly optimizing domain discriminators at different layers with the optimal weights, we obtain a lower bound to the $\cH$-divergence, which is described as follows,
\begin{prop}
The Eqn.~(\ref{eqn:dann}) can be written as,
\begin{eqnarray}
\min_{\mathbf{\bm{\theta}}}-1\cdot\min_{\w}\frac{1}{N}\sum_{i=1}^N\cL_{D}(D(F(\x_i; \bm{\theta}); \w), d_i).\nonumber
\end{eqnarray}
The optimum of Eqn.~(\ref{eqn:ca}) is smaller than or equal to the optimum of the above equation.
\end{prop}
\begin{proof}
The objective in the above equation can be seen as a special case of Eqn.~(\ref{eqn:ca}) when setting $\lambda_m = -1$, and $\lambda_l = 0$  for $l = 1, \ldots, m -1$. Thus, by optimizing Eqn.~(\ref{eqn:ca}), we expect to obtain to obtain a lower minimum than the optimum of the above equation. We complete the proof.
\end{proof}
By introducing multiple discriminators at different layers, and jointly optimizing all discriminators and their ensemble weights, we expect to obtain a lower domain discrepancy than DANN, thus we can better reduce the distribution mismatch between two domains. 

\vspace{-2mm}
\section{Self-Paced CAN (SPCAN)}
To effectively use the unlabeled target samples for learning better representations, in this section, we further propose a sample selection strategy to progressively select a set of pseudo-labeled target samples for training the model. A simple strategy based on thresholding was proposed in our preliminary work~\cite{zhang2018collaborative}. 
% \textcolor{blue}{
% To avoid introducing many pre-defined hyper-parameters for target samples selection as in \cite{zhang2018collaborative},
% \wenli{iCAN has more hyper-parameters than SPCAN? Please double check this argument. How many hyper-parameters does iCAN have, and how many does SPCAN have? A:when learning sample selection part, iCAN have alpha and beta, and the weight curve is tuned.}
Inspired by the self-paced learning \cite{kumar2010self}, we present a more effective approach to select pseudo-labeled samples, referred to as Self-Paced CAN (SPCAN). %in this section
Again, we focus on the image classification task in Sections \ref{sec:spl}, \ref{sec:sscspcan}, \ref{sec:ssdspcan}, \ref{sec:spcansummary}, and will discuss how to extend our work for video action recognition in Section \ref{sec:cospcan}.

\vspace{-1mm}
\subsection{Unsupervised Domain Adaptation with Self-Paced Learning}
\label{sec:spl}
Self-paced learning was originally designed for the supervised learning task. Let us denote the training data as $\X=\{(\x_i,y_i)|_{i=1}^{N}\}$, where $\x_i$ is the $i$-th training sample, $y_i$ is its corresponding label, and $N$ is the total number of samples. We also denote the prediction function as $g(\x_i,\w_{sp})$  with $\w_{sp}$ being the parameters to be learnt.  Let $\cL(y_i,g(\x_i,\w_{sp}))$ be the loss of the $i$-th sample. The goal of self-paced learning is to jointly learn the parameter $\w_{sp}$ and the latent weight variable $\mathbf{s} = \{s_i|_{i=1}^{N}\}$ for each sample by minimizing the objective function:
\begin{eqnarray}
\begin{aligned}
\label{eqn:spl}
\min_{\w_{sp} \in \cW,\mathbf{s}\in[0,1]^n} % \E(\w_{sp},\mathbf{s};T_{sp}) =&
\sum_{i=1}^{N}s_i\mathcal{L}(y_i,g(\x_i,\w_{sp})) -\sum_{i=1}^N s_i T_{sp},
\end{aligned} 
\end{eqnarray}
where $T_{sp}$ is the threshold (\ie, the age parameter) that controls the learning pace, and $-\sum_{i=1}^N s_i T_{sp} $ 
%$f(s_i, T_{sp})$ represents , which 
is the self-paced regularizer.

Using the alternating convex optimization strategy \cite{bazaraa2013nonlinear,kumar2010self}, for fixed $\w_{sp}$, the solution $\mathbf{s}^* = \{s_i^*|_{i=1}^{N}\}$ can be easily calculated as,
\begin{eqnarray}
\begin{aligned}
\label{eqn:spv}
s_i^* &=\begin{cases}
1, & \mathcal{L}(y_i,g(\x_i,\w_{sp}))< T_{sp}.\\
0, & \text{otherwise}.
\end{cases}
\end{aligned}
\end{eqnarray}

Intuitively, when updating the weight variable $\mathbf{s}$ with fixed $\w_{sp}$, a sample is selected ($s_i^* = 1$) when its loss is smaller than the threshold age parameter $T_{sp}$, in which the selected sample is also regarded as an ``easy'' sample for small $T_{sp}$. We will gradually select more samples in an easy-to-hard fashion, by gradually increasing the age parameter $T_{sp}$. 

In unsupervised domain adaptation, the goal is to train a model for the target domain. Thus the easiness of a training sample should be dependent on two aspects: 1) easy for classification, which shares the same spirit as self-paced learning; 2) easy for adaptation, which means the samples should be confidently discriminated to their corresponding domain, and their gradients shoule be further reversed to correctly update the feature extractor. From the above motivations, we respectively design a self-paced sample selection module based on the image classifier and the domain classifier, and integrate the selected pseudo-labeled target samples for boosting the model performance. The detail of the two selection schemes are provided below. 

% Challenges are two-fold. 1) only have unlabeled samples, we thus use the prediction confidence.
%  and weighting 
\vspace{-1mm}
\subsection{Image Classifier Based Sample Selection (CSS)}
\label{sec:sscspcan}
% Inspired by the self-paced learning selection strategy, we attempt to use more pseudo-labeled target samples in training the model by following the gradual easy-to-hard paradigm, which leads to SPCAN. In unsupervised domain adaptation, we are not provided with the ground-truth label of the target data. Therefore, the loss required in (\ref{eqn:spv}) is not available. Instead of using threshold for loss, we empirically utilize the prediction probability as the selection criteria. We first pre-train an initial CAN model, then for SPCAN, we separate the selection process in two parts, including the image classification part and the domain classification part, both containing their corresponding classifiers and features respectively. The details of the two parts will be described below. For the first image classifier part, we select the target samples according to the classification confidence scores.

% \vspace{-1mm}
\subsubsection{Image Classification Confidence Score} 
\label{sec:spcan_ccs}
We follow the idea of self-paced learning to gradually select easy target samples determined by the image classifier, which is referred to as the image classifier based sample selection (CSS) module. However, the major challenge in unsupervised domain adaptation is that we have only unlabeled samples in the target domain. Since the ground-truth labels for these unlabeled samples are absent, we cannot measure the easiness of these samples based on the classification loss as in the original self-pace learning work\cite{kumar2010self}. Thus, we instead determine the easiness of target samples according to their image classification confidence scores.

Specifically, let us denote $\{p_{cl}(\x^t_i)|_{cl=1}^{N_{cl}}\}$ as the output from the softmax layer of the image classifier, in which each $p_{cl}(\x^t_i)$ is the probability that $\x^t_i$ belongs to the $cl$-th category, and $N_{cl}$ is the total number of categories. The pseudo-label $\tilde{y}^t_i$ of $\x^t_i$ can be obtained by choosing the category with the highest probability, \textit{i.e.}, $\tilde{y}^t_i = \arg\max_{cl}p_{cl}(\x^t_i)$. We refer to the probability $p_{\tilde{y}^t_i}(\x^t_i)$ as the \textit{iamge classification confidence score}. Intuitively, the higher the classification confidence score is, the more likely the target sample is correctly predicted. In other words, we assume this sample would be relatively easy for classification. 
%A pseudo-labeled  target sample is denoted by $(\x^t_i, \tilde{y}^t_i)$.

%\textcolor{blue}{class c is conflicted with classfier c, is cl ok?}

\vspace{-1mm}
\subsubsection{Selection Strategy}
\label{sec:spcan_strategy}
Considering that the predicted pseudo-labels for some target samples may be incorrect, we only select the highly confident samples to retrain our model for better performance.
In our preliminary conference paper~\cite{zhang2018collaborative}, we presented a selection strategy based on thresholding during the whole training process. However, prediction of the CNN model often fluctuates during the training procedure due to usage of mini-batches. It becomes even more obvious when additionally incorporating adversarial training in the unsupervised domain adaptation tasks. Nevertheless, we note that ``easy" samples and ··difficult“ samples are actually decided relatively. Thus, we instead turn to select easiest samples with a ranking strategy.

Specifically, at each epoch, we use the image classifier to predict all the target samples. Then the classification confidence scores are used for sorting the target samples. After the sorting procedure, we select $r_c\%$ pseudo-labeled target samples with the highest classification confidence scores, where $r_c \in [0,\ 100]$ is a proportion parameter updated at each epoch to control the learning pace. $r_c$ plays a similar role to the threshold age parameter $T_{sp}$ as in self-paced learning. To ensure the quality of the selected pseudo-labeled samples, we dynamically adjust the proportion parameter $r_c$ based on the classification accuracy of the learnt model, which is validated by using the source domain samples. When the classification accuracy drops in two consecutive epochs, the proportion parameter $r_c$ is decreased to prevent model collapse by using low quality pseudo-labeled target samples. Otherwise, $r_c$ would be increased. %until all samples are selected \wenli{Double check this sentence}. 

Formally, let $A_i$ be the average classification accuracy over all the
source samples at the $i$-th epoch,
% \textcolor{blue}{$(i > r_0 T)$}, 
and $\bar{A}_i$ be the average accuracy of the first $i$ epochs (\ie, $\bar{A}_i = \frac{1}{i}\sum_{j=1}^iA_j$).  
% Suppose we have 
For the total $T$ training epochs, the proportion parameter $r_c$ at the $e$-th epoch is then adjusted according to,
% \wenli{I revised this equation. With $r_0$, it is possible for $r$ to exceed 1 in the original equation.}
\begin{eqnarray}
\begin{aligned}
\label{eqn:re}
r_c = & \;\frac{\sum_{i=1}^{e}\eta_i}{T}, \\
\label{eqn:rank}
\eta_i = & \begin{cases}
-1, & \text{if $A_i < \bar{A}_i$ and $A_{i-1}<\bar{A}_{i-1}$},\\
1, & \text{otherwise}.
\end{cases}
\end{aligned}
\end{eqnarray}
% \textcolor{blue}{where $\alpha$ is a hyper-parameter that controls the learning proportion of the selected pseudo-labeled target samples.}
% r_0 + (1-r_0), 

% As the result, let us denote $\mathbb{S}^c$ as the set of selected pseudo-labeled samples using confidence score for re-training image classifier.}

%\textcolor{blue}
% {
Let us denote $\mathbb{S}^c$ as the set of selected pseudo-labeled target samples by using the classification confidence score. Then, after selecting pseudo-labeled target samples, we add them into the training set, and re-train the image classifier in the next epoch. Moreover, considering that the low-confident samples are more likely to have incorrect pseudo labels, we assign different weights to different target samples according to their classification probabilities, \ie, $w^c(\x_i^t) = p_{\tilde{{y^t_i}}}(\x^{t}_{i})$. For convenience of presentation, let us define an indicator function $s^c(\x_i^t)$ which produces $1$ if $\x_i^t$ is selected (\ie, $\x_i^t \in \mathbb{S}^c $) and $0$ otherwise. Then, the loss of the image classifier for the selected pseudo-labeled target samples can be written as, 
\begin{eqnarray}
\label{eqn:target_c}
\cL_{tar}^c = \frac{1}{N^t}\sum_{i=1}^{N^t}s^c(\x_i^t)w^c(\x_i^t) \cL_{C}(C(F(\x^t_i; \Theta_F);\c), \tilde{y}^t_i),
\end{eqnarray}
where the pseudo-label $\tilde{y}^t_i$ is treated as the label of the sample $\x_i^t$ for computing the loss $\cL_{C}(\cdot,\cdot)$ and learning the parameters $\Theta_F$.

\vspace{-1mm}
\subsection{Domain Classifier Based Sample Selection (DSS)}
\label{sec:ssdspcan}
\subsubsection{Domain Classification Confidence Score}
We also gradually improve the domain classifiers by re-training the model with the selected target samples based on the domain classification confidence scores in an easy to hard manner, which is referred to as domain classifier based sample selection (DSS) module.

For a target sample $\x^t_i$, let us denote $d(\x^t_i)$ as the predicted probability from the last domain classifier trained in our CAN model, where the domain label is equal to $1$ when the sample is from the source domain, and $0$ when the sample is from the target domain. We define $1 - d(\x^t_i)$ as the \textit{domain classification confidence score}.
% So with using the reverse gradient layer, when the target domain discriminative score, we use the absolute 0.5 to measure whether the sample is easy or hard to be distinguished by the domain classifier(\ie, $d(\x^t_i) $), which is also called \textit{target domain discriminative score}. 
If the domain classification confidence score is higher, it is easier for the target domain sample to be classified by the domain classifier, which should be selected and reweighed with higher weights when learning the model.

\vspace{-1mm}
\subsubsection{Selection Strategy}
Similarly to our approach for the image classifier, we also define a proportion parameter $r_d$ for selecting the samples with the highest domain classification confidence scores. The proportion parameter $r_d$ is gradually increased at each epoch to include more samples for learning the domain classifier. At the $e$-th epoch, $r_d$ is adjusted as, 
\begin{eqnarray}
\begin{aligned}
\label{eqn:rd}
r_d &= \frac{e}{T},
\label{eqn:rank_d}
\end{aligned}
\end{eqnarray}
where $T$ is the total number of epochs.

Let us denote $\mathbb{S}^d$ as the set of target samples  selected by using their domain classification confidence scores. We also apply the weights on the selected target samples based on their predicted domain probabilities, which is defined as
\begin{eqnarray}
\begin{aligned}
\label{eqn:wd}
w^d(\x_i^t) = 2 (1 - d(\x^{t}_{i})).\\
\end{aligned}
\end{eqnarray}
In this way, when the CNN parameters are learnt by using the gradient reversal layer, 
lower weights (\ie, $w^d(\x_i^t) < 1$) are assigned to the selected pseudo-labeled target samples which are already predicted to be ``source'', and higher weights (\ie, $w^d(\x_i^t) > 1$) are assigned to the samples which are still predicted to be ``target'', to encourage more target samples to be classified as ``source''. 
%In contrast, when learning the domain discriminator, the discriminator is learnt to accurately distinguish source domain samples from target domain samples.

% The effectiveness of this weighting function will be discussed in Section \ref{sec:modelanalysis}.
% The reason behind this is that according to the reverse gradient of domain loss in \cite{ganin2015unsupervised, ganin2016domain}, the target samples that are predicted  we think the wrongly predicted target 
% In this way, as the training progress goes on, for domain discriminator part, the selection threshold $n_d$ is become smaller and more domain indistinguishable samples with smaller weights are utilized to re-train the network, which follows the easy-to-hard training strategy. 

Empirically, we also find that the results are improved if we re-train the domain classifier by including the selected pseudo-labeled target samples in the set $\mathbb{S}^c$, which were used for training the image classifier. A possible explanation is that we emphasize the importance of high-confident samples when learning the domain classifiers, and also further reduce the data distribution mismatch for these samples. As a result, the loss of the selected pseudo-labeled target samples for the domain classifier part can be written as,
% \vspace*{-4pt}
\begin{eqnarray}
\label{eqn:target_d}
\!\!\!\!\!\!\!\!\min_{\Theta_F,\W,\lambda_l \in \Lambda}\!\! \cL_{tar}^d \!\!\!\!\!&=&\!\!\!\!\! \frac{1}{N^t}\sum_{i=1}^{N^t}((s^c(\x_i^t)+s^d(\x_i^t))w^d(\x_i^t) \cL_{CA}),
\end{eqnarray}
where $s^c(\x_i^t)$ is the selection indicator from the image classifier defined in \ref{sec:spcan_strategy}, and $s^d(\x_i^t)$ is the selection indicator from the domain classifier, which equals to $1$ if $\x_i^t$ is selected (\ie, $\x_i^t \in \mathbb{S}^d $)  and $0$ otherwise. 

% The details about the sample selection for the domain classifier part will be discussed in section \ref{sec:modelanalysis}.
% \lambda_l \cL_{D}(D(F(\x^t_i; \Theta_F);\w), d_i)
% \sum_{l=1}^{m}
% where only the feature layers and domain classifiers are updated without the $\lambda_l$s in Eqn (\ref{eqn:ca}).
\begin{algorithm}[t]
\caption{Self-Paced Collaborative and Adversarial Network (SPCAN) for image classification.}\label{algo:spcan}
\begin{algorithmic}[1]

\State{\textbf{Input}: source domain labeled samples $\{(x^s_{i},y^s_{i})|_{i=1}^{N^s}\}$, target domain unlabeled samples $\{(x^{t}_i)|_{i=1}^{N^t}\}$. }
\Statex \emph{\textbf{Stage-1}}:
    \State{Train an initial CAN model by optimizing Eqn. (\ref{eqn:ca}).}
\Statex \emph{\textbf{Stage-2}}:
\Loop~until~$max\_epoch$~is reached:
    \State{Select the pseudo-labeled target samples based on $s^c(\x_i^t)$ by using the image classifier, and select the target samples based on $s^d(\x_i^t)$ by using the latest domain classifier, which respectively construct the set $\mathbb{S}^c$ containing target samples with both pseudo-labels $\tilde{y_i^t}$ and domain labels ``0'', and the set $\mathbb{S}^d$ containing target samples with domain labels ``0'' only.}
    \State{Calculate the weights $w^c(\x_i^t)$ and $w^d(\x_i^t)$ for the image classifier and the last domain classifier, respectively.}
    \State{Update the image classifier $\c$ and the parameters $\Theta$ from the CNN feature extractor based on the classification loss of the source samples using Eqn. (\ref{eqn:src}) and the selected pseudo-labeled target samples in the set $\mathbb{S}^c$ using Eqn. (\ref{eqn:target_c}).}
    \State{Update the domain classfiers $\W$, the weights of domain classifiers $\Lambda$, and the parameters $\Theta$ from the CNN feature extractor based on the loss of all the samples from both domains using Eqn. (\ref{eqn:ca}) and also the loss of the selected target samples in the set $\mathbb{S}^c$ and $\mathbb{S}^d$ using Eqn. (\ref{eqn:target_d}).}
\EndLoop
\end{algorithmic}
\end{algorithm}
\vspace{-1mm}
\subsection{SPCAN Summary}
\label{sec:spcansummary}
The objective of our SPCAN is to jointly optimize four losses, which include $\cL_{CA}$ defined in Eqn. (\ref{eqn:ca}), $\cL_{src}$ defined in Eqn. (\ref{eqn:src}), $\cL_{tar}^{c}$ defined in Eqn. (\ref{eqn:target_c}), and $\cL_{tar}^{d}$ defined in Eqn. (\ref{eqn:target_d}). The selected pseudo labelled target samples are only used in $\cL_{tar}^{c}$ and
$\cL_{tar}^{d}$. All labelled source samples are only used in $\cL_{src}$, while all unlabelled samples from both domain are used in $\cL_{CA}$. The total objective of SPCAN can be written as follow,
\begin{eqnarray}
\label{eqn:spcan}
\!\!\!\!\!\!\!\!\min_{\Theta_F,\c,\W,z,\lambda_l \in \Lambda} \!\!\cL_{SPCAN} \!\!\!\!\!\!&=& \!\!\!\!\!\!\!\,\cL_{src} + \cL_{tar}^{c} + \alpha(\cL_{CA} + \,\cL_{tar}^{d}).
\end{eqnarray}

The whole pipeline of SPCAN is illustrated in Fig~\ref{fig:Overview}, where the left part is the CAN model described in Section \ref{sec:can}, and the right part is the self-paced sample selection module. We iteratively select pseudo-labeled target samples and add them into the training set for training the CAN model in the next epoch.  We also depict the training process for SPCAN in Algorithm \ref{algo:spcan}. 
% Similar to our preliminary work iCAN \cite{zhang2018collaborative},
The training process for SPCAN includes two stages. For the first stage, we train an initial CAN model by optimizing Eqn. (\ref{eqn:ca}) for a small number of epochs. With this initial training strategy, the image classifier and domain classifiers are able to provide reasonable prediction results so that we can use them for selecting pseudo-labeled target samples. For the second stage, we separately select and reweigh different pseudo-labeled target samples for the image classifier and the domain classifiers, and use the gradients from their corresponding losses to update the classifiers and CNN parameters respectively.

\vspace{-2mm}
\subsection{Two-stream SPCAN for Video Action Recognition}
% \vspace{-1mm}
\label{sec:cospcan}
% \wenli{I feel it is better to keep exchange stratey only, making the contribution simple and clear. I do not see a clear motivation for fusion strategy.}
We further extend our SPCAN model to the video action recognition task under the unsupervised domain adaptation setting. In the video action recognition task, both spatial information and temporal information are important for classifying actions. Many Two-Stream ConvNets were proposed \cite{simonyan2014two, wang2016temporal}, in which two convolutional networks trained for the RGB stream and optical flow stream are finally fused for action recognition. Inspired by the co-training approach \cite{blum1998combining}, we further exploit the complementary information in two streams, and encourage the networks of two streams to help each other in the self-paced learning procedure, which is referred to as Two-stream SPCAN (TS-SPCAN).

\begin{figure}[H]
\begin{center}
\includegraphics[width=0.9\linewidth]{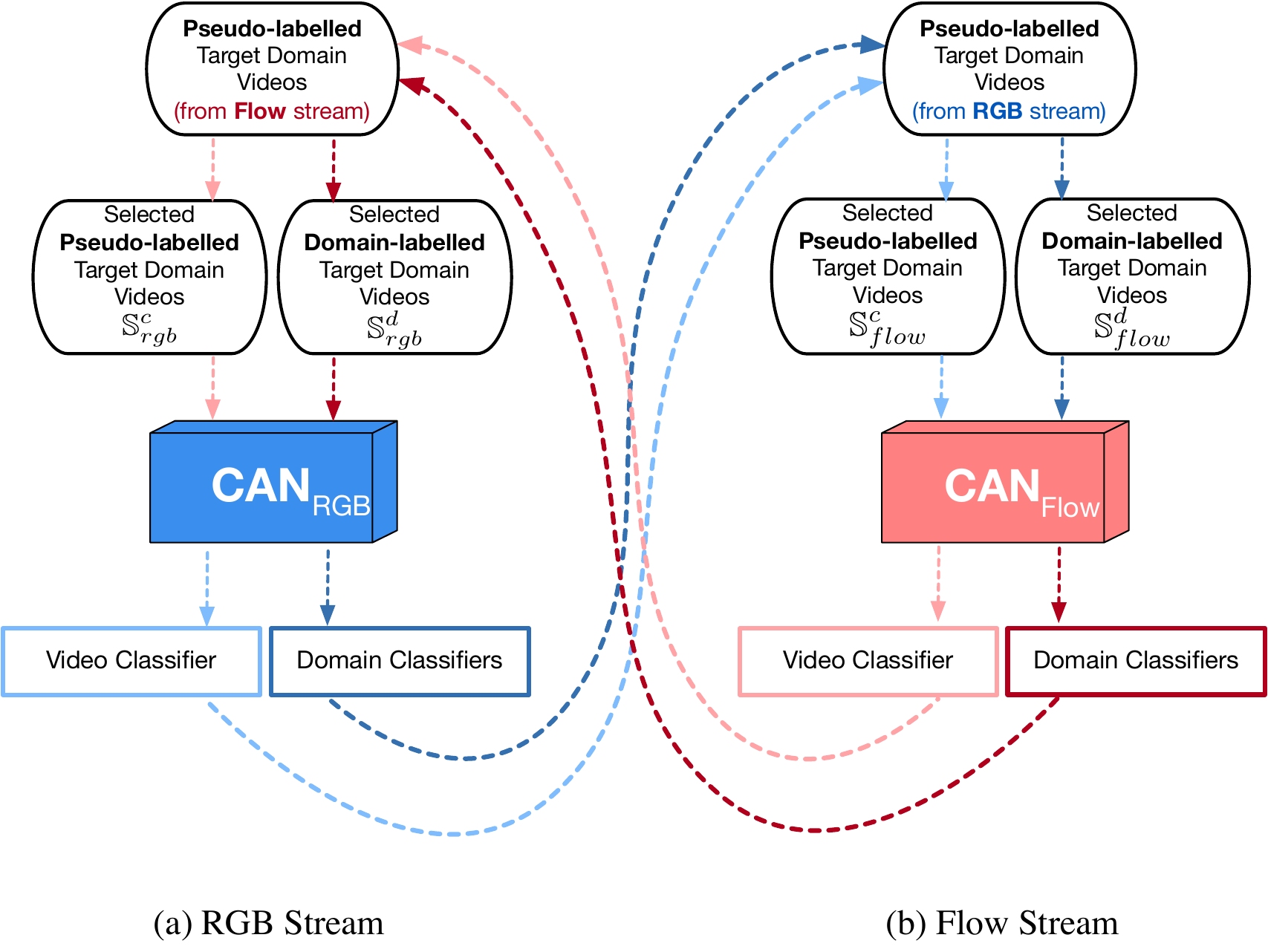}
\end{center}
\vspace{-10pt}
   \caption{Overview of the pseudo-labeled target sample selection process in our TS-SPCAN. We use the video classifier and last domain classifier of one stream (RGB/Flow) to select the target samples for another stream (Flow/RGB). $\mathbb{S}^c_{rgb}$ and $\mathbb{S}^c_{flow}$ are the selected pseudo-labelled target domain samples by using the video classifiers of the RGB and flow stream, respectively. $\mathbb{S}^d_{rgb}$ and $\mathbb{S}^d_{flow}$ are the selected target domain samples by using the last domain classifiers of the RGB and flow streams, respectively.} %$C_{rgb}$ and $C_{flow}$ are the video classifiers for the RGB and flow streams, respectively. $D_{rgb}$ and $D_{flow}$ are the domain classifiers for the RGB and flow streams, respectively.
\vspace{-2pt}
\label{fig:cospcan}
\end{figure}

For the training procedure of TS-SPCAN, we train a CAN model for the RGB stream and train another CAN model for the optical flow stream. The total loss for one stream is then:
\begin{eqnarray}
\label{eqn:cospcan}
\cL_{TS-SPCAN}^M \!= \!\!\, \cL_{src}^M + \cL_{tar}^{c, M} + \alpha(\cL_{CA}^M + \cL_{tar}^{d,M}).
\end{eqnarray}
The loss function above is very similar to the objective of SPCAN in Eqn. (\ref{eqn:spcan}). Since there are two streams, we use $M$ as the index, \ie, $M=flow$ for the optical flow stream and $M=rgb$ for the RGB stream.

For each stream $M$, we have a cross-entropy loss $\cL_{src}^M$ for the video classification task as in Eqn. (\ref{eqn:src}) and the loss $\cL_{CA}^M$ for the collaborative and adversarial loss as in Eqn. (\ref{eqn:ca}). 
For the loss functions $\cL_{tar}^{c, M}$ and $\cL_{tar}^{d, M}$ that utilize psuedo-labeled data, we also calculate the sample selection indicators and sample weights, but in a way different from SPCAN.
%using the same functions as in SPCAN, similar to the procedure in line 4 and line 5 of Algorithm \ref{algo:spcan}. 

%Formally, let us denote $p_{i}^{rgb}$ (\resp, $p_{i}^{flow}$) as the score for the $i$-th video $V_{i}$ predicted by the networks from the RGB (\resp, optical flow) stream.

Specifically, take the RGB stream as an example, the sample selection indicator $s^c_{rgb}(\cdot)$ and the sample weight $w^c_{rgb}(\cdot)$ in $\cL_{tar}^{c, rgb}$ are selected and calculated based on the video classifier from another stream (\ie, optical flow) with the same self-paced selection strategy as in (\ref{eqn:target_c}). Similarly, the sample selection indicator $s^d_{rgb}(\cdot)$ and the sample weight $w^d_{rgb}(\cdot)$ in $\cL_{tar}^{d, rgb}$ are selected and calculated based on the last domain classifier from the optical flow stream with the same strategy as in (\ref{eqn:target_d}).

%Specifically, different from (\ref{eqn:target_c}), the pseudo-labelled samples selected from one stream (\eg, RGB) are used to re-train the feature extractor and the video classifier of another stream (\eg, optical flow) in $\cL_{tar}^{c, M}$. Similarly, samples selected from one stream (\eg, RGB) are also used for re-training the feature extractor and the domain classifiers of the another stream (\eg, optical flow) in $\cL_{tar}^{d, M}$.

%and $\cL_{tar}^{d, M}$ can be defined similar to (\ref{eqn:target_c}) and (\ref{eqn:target_d}) by replacing the indicators with the cross-stream selection indicators defined as follows:
% \begin{eqnarray}
% \label{eqn:cospcan_sc1}
% \begin{aligned}
% s^c_{rgb}(V_{i}) &=\begin{cases}
% 1, & p^{flow}_i> T^c_{flow}.\\
% 0, & \text{otherwise}.
% \end{cases}\\
% s^c_{flow}(V_{i}) &=\begin{cases}
% 1, & p^{rgb}_i> T^c_{rgb}.\\
% 0, & \text{otherwise}.
% \end{cases}
% \end{aligned}
% \end{eqnarray}
% where $s^c_{rgb}(V_{i})$ and $s^c_{flow}(V_{i})$ are the sample selection indicators for the video classifiers of RGB and optical flow streams, respectively. $T^c_{rgb}$ and $T^c_{flow}$ are the sample selection thresholds %by calculating $r_c$ 
% as in SPCAN by using the classification confidence scores from the RGB stream and the optical flow stream, respectively.  The two sets of selected samples by using the two indicators for the RGB and optical flow streams are then denoted as $\mathbb{S}_{rgb}^{c}$ and $\mathbb{S}_{flow}^{c}$, respectively. The similar strategy is also used for the sample selection indicators for the domain classifiers of the two streams.

\textbf{Analysis.} In SPCAN, the sample selection indicators and sample weights calculated from one stream are used by the stream itself.
Moreover, in TS-SPCAN, inspired by co-training \cite{blum1998combining}, the selection indicators and weights calculated from one stream are used for training the networks of another stream. In this way, we exchange the selected target samples and their corresponding weights between the RGB and flow streams, which can exchange the information of the two streams so that the  deep models of the two streams help each other during the learning process.

%as follows:
%\begin{eqnarray}
%\label{eqn:target_cd2}
%\cL_{tar}^{c,m} = \frac{1}{N^t}\sum_{i=1}^{N^t}s^{c}_m(\x_i^t)w^c_m(\x_i^t) \cL_{C}(C(F(\x^t_i; \Theta_F);\c), \tilde{y}^t_i), \\
%\cL_{tar}^{d,m}= \frac{1}{N^t}\sum_{i=1}^{N^t}((s^c_m(\x_i^t)+s^d_m(\x_i^t))w^d_m(\x_i^t) \cL_{CA})
%\end{eqnarray}
%\textcolor{blue}{ ??? we did not mentioned $T^c$ before. }
% , and  $T^c_{rgb}$ and $T^c_{flow}$ are the thresholds calcualted following Eqn. (\ref{eqn:thresholdc}). 

%Similarly, we also calculate the sample indicator and weights for domain classifiers on two streams using the target domain discriminative score, and exchange them for training domain classifiers on different streams. The two sets of selected samples using the two discriminator indicators for RGB and optical flow streams are denoted as $\mathbb{S}_{rgb}^{d}$ and $\mathbb{S}_{flow}^{d}$, respectively.

The detailed exchange process for the selected pseudo-labeled target samples between the RGB stream and the optical flow stream is illustrated in Fig. (\ref{fig:cospcan}).  %take RGB stream as an example?,
For example, firstly in the % first
RGB stream, we use the video classifier and the last domain classifier of this individual stream to predict each video and use the classification confidence scores and the domain classification confidence scores of this individual stream to select pseudo-labeled target samples (\ie, the sets $\mathbb{S}_{rgb}^{c}$ and $\mathbb{S}_{rgb}^{d}$) for the optical flow stream. Then, for the %second
optical flow stream, we use the target samples from the RGB stream to re-train and improve the deep model of the optical flow stream, and use the updated model to predict each video and select target samples based on the classification confidence scores and the domain classification confidence scores. Again, the selected target samples from the optical flow stream are used to re-train the deep model of the RGB stream in the next training epoch. This selection and re-training procedure is iteratively performed until the models of the two streams converge.

\vspace{-2mm}
\section{Experiments}
\label{sec:exp}
In this section, we evaluate our proposed methods for the unsupervised domain adaptation problem under two different tasks, object recognition and video action recognition. Then, we also investigate different components of our methods in details.

% The source code of our preliminary work iCAN \cite{zhang2018collaborative} is available online\footnote{https://github.com/zhangweichen2006/iCAN}. 

\vspace{-1mm}
\subsection{Experiments for the object recognition tasks}

% \vspace{-1mm}
\subsubsection{Datasets}

For the object recognition tasks, we evaluate our methods on three benchmark datasets, Office-31, ImageCLEF-DA and VISDA-2017.

The Office-31 dataset \cite{saenko2010adapting} is an object recognition benchmark dataset for evaluating different domain adaptation approaches, which contains 4,110 images from 31 classes. It has three domains: Amazon (A), Webcam (W) and DSLR (D).

The ImageCLEF-DA dataset \cite{long2017jan} is built for the ImageCLEF 2014 domain adaptation challenge\footnote{http://imageclef.org/2014/adaptation}. It contains 4 subsets, including Caltech-256 (C), ImageNet ILSVRC 2012 (I), Bing (B) and Pascal VOC 2012 (P). Each of the subset contains 12 classes and each class has 50 images, which results in a total number of 600 images for one subset. For each dataset, we follow \cite{long2015learning, long2017jan} to utilize the common evaluation protocol under all 6 domain adaptation settings.

VISDA-2017 \cite{peng2018visda} is also a domain adaptation  challenge dataset\footnote{http://ai.bu.edu/visda-2017/}. It focuses on the domain adaptation from synthetic images to real-world images. The dataset contains two domains. The source domain is the synthetic domain (\ie the training set), which contains rendered 3D CAD model images from different angles and lightning conditions. The target domain is the real-world domain, which contains natural images cropped from the COCO dataset\cite{lin2014microsoft} (\ie the validation set). The whole dataset contains 152,397 images for the training set and 55,388 images for the validation set, which are all from 12 categories. 
%We use two settings to report the results of the VISDA-2017 validation set. The first setting (\ie, VISDA Setting 1) uses all images from the training set as the source domain for different methods. Due to a few images from the same 3D model and similar viewpoints are redundant for model training, the second setting (\ie, VISDA Setting 2) randomly uses 1000 images from each class to form a small subset of VISDA-2017, which are used as the source domain for different methods. %We observe that the results of our methods CAN and SPCAN using a small proportion of training data are already sufficient to demonstrate the effectiveness of our SPCAN, which also outperforms other state-of-the-art methods (\eg, \cite{long2017jan,sankaranarayanan2017generate,pinheiro2018unsupervised,saito2017adversarial}) by using the whole training set.

% All three datasets focus on distinct adaptation aspects. Office-31 focuses on single-to-single object adaptation with different indoor background environments. ImageCLEF-DA focuses on real world object adaptation, which might both contain multiple objects in real environments and requires to find the most important ones. VISDA-2017 is a synthetic-to-real dataset, which the training dataset contains single CAD model object images, while the validation and test datasets are both real world images. 

\vspace{-1mm}
\subsubsection{Implementation Details}
We implement our proposed methods based on the PyTorch framework, which are made publicly available at \url{https://github.com/zhangweichen2006/SPCAN}.

We use the ResNet50 model pre-trained based on the ImageNet dataset as the backbone network. %, followed by the domain discriminators and the image classifier. 
The learning rate for the discriminators and image classifier are set as 10 times of the backbone network(\ie, the feature extraction layers).
In the network, the feature extraction layers are grouped into four blocks. We add four domain discriminators after
% 10th(res2b), 22th(res3d), 40th(res4f) and 49th layer(res5c),% \wenli{it is better to specify the name of the layers instead of using ``10th"..., for example, ``res2c", ``res3c", ...} receptive fields
the 2$^{nd}$, 3$^{rd}$, 4$^{th}$ and 5$^{th}$ pooling layers (\ie, after the 10$^{th}$(res2b), the 22$^{th}$(res3d), the 40$^{th}$(res4f) and the 49$^{th}$ layer(res5c) for ResNet50). The image classifier is placed at the last layer (\ie, after the 5$^{th}$ pooling layer).
We use stochastic gradient descent(SGD) for optimization.

The learning rate of SPCAN is set as $0.0015$ initially and decreases gradually after each iteration. We use the same INV learning rate decrease strategy as in DANN \cite{ganin2016domain}. Following the setting in DANN \cite{ganin2015unsupervised, ganin2016domain}, an adaptation factor is used for controlling the learning rate of the domain discriminator. In our experiment, we set the batch size, momentum, and weight decay as 16, 0.9 and $3\times10^{-4}$, respectively. We set the trade-off parameter $\alpha=0.4$, and fix $\lambda_m=-2$ in order to have more stabilized training process for all tasks. 

In each domain adaptation task, we utilize all samples from both source and target domains. Due to different dataset sizes in different tasks, we utilize different total training epoch lengths and different number of pseudo-labeled target samples in one batch. For the training epoch length, we decide it based on the fixed total iteration for different tasks, similar to other settings in \cite{ganin2016domain,long2016unsupervised,long2017jan}. 

Similar to our preliminary iCAN \cite{zhang2018collaborative} method, for image classification, we use the same number of labelled and unlabelled samples in one batch. In one training epoch, let us denote $N^s$ and $N^p$ as the total number of source samples and the number of target samples selected by the CSS module, respectively. Let us also denote $N^t$ and $N^q$ as the total number of target samples, and the number of target samples selected by the DSS module, respectively. For convenience of presentation, we denote $\beta = \frac{N^p}{N^s+N^p}$ as the ratio of the number of target samples selected by the CSS module over the total number of source samples and target samples selected by the CSS module.  To train the image classifier and domain classifiers with all selected target samples, we balance the number of target samples selected by the CSS module and the DSS module with the number of source samples in one mini-batch. For the image classification part, when one mini-batch has eight samples, we use $8(1-\beta)$ source samples with the category label and $8\beta$ pseudo-labelled target samples selected by the CSS module to train the network. For the domain classification part, when one mini-batch has sixteen samples, we use $8$ source samples, $8\beta$ target samples selected by the CSS module, $8(1-\beta)\frac{N^q}{N^t}$ target samples selected by the DSS module, and the remaining samples are randomly selected target samples.  The domain labels for the source samples and the target samples are ``1'' and ``0'', respectively. In this way, the numbers of training samples from the source and target domains are the same, which is helpful for learning the model that is balanced for both source and target domains.

% \subsubsection{State-of-the-art Approaches}
\vspace{-1mm}
\subsubsection{Experimental Results}

We compare our method with the basic deep learning method (ResNet50) and the existing deep domain adaptation learning methods based on ResNet50. 
For the basic deep learning method, we use only source samples to fine-tune the ResNet50 \cite{he2016deep} model that is pre-trained based on the ImageNet dataset. For deep transfer learning methods, we report the results of Deep Adaptation Network (DAN) \cite{long2015learning}, Residual Transfer Network (RTN) \cite{long2016unsupervised}, Joint Adaptation Network (JAN) \cite{long2017jan}, Multi-Adversarial Domain Adaptation(MADA)\cite{pei2018multi}, Similarity Learning (SimNet)\cite{pinheiro2018unsupervised}, Generate To Adapt (GTA)\cite{sankaranarayanan2017generate}, Deep Adversarial Attention Alignment (DAAA) \cite{kang2018deep},  Adversarial Dropout Regularization (ADR)\cite{saito2017adversarial},  and our work CAN and SPCAN.
% Conditional Domain Adversarial Network (CDAN) \cite{long2018conditional}, 
In addition, we also report the results of Domain Adversarial Training of Neural Network (DANN) \cite{ganin2016domain} based on our own implementation.

% The experimental results of these models are achieved without using any model ensembling or multi-scale inference method. %\wenli{Is any of those baseline methods using ensembling or multi-scale inference? No, but self-ensembling(https://arxiv.org/pdf/1706.05208.pdf) used ensembling which gets better results for VISDA2017 dataset(82.8)}

\begin{table}
\captionsetup{justification=centering, labelsep=newline, font=sf, textfont=footnotesize}
\setlength\tabcolsep{1.7pt}\small
\begin{center}
\caption{Accuracies (\%)  of different methods on the Office-31 dataset.}
\label{table:officemain}
\vspace{-1mm}
\begin{tabular}{c c c c c c c c}
\toprule
\textbf{Methods} & A$\rightarrow$W & W$\rightarrow$A & A$\rightarrow$D & D$\rightarrow$A & W$\rightarrow$D & D$\rightarrow$W & Avg.\\
\midrule
ResNet50\cite{he2016deep} & 73.5 & 59.8 & 76.5 & 56.7 & 99.0 & 93.6 & 76.5\\
%DDC\cite{tzeng2014deep}  & 76.0 & 63.7 & 77.5 & 67.0 & 98.2 & 94.8 & 79.5\\
DAN\cite{long2015learning} & 80.5 & 62.8 & 78.6 & 63.6 & 99.6 & 97.1 & 80.4\\
RTN\cite{long2016unsupervised} & 84.5 & 64.8 & 77.5 & 66.2 & 99.4 & 96.8 & 81.6\\
DANN\cite{ganin2016domain} & 79.3 & 63.2 & 80.7 & 65.3 & 99.6 & 97.3 & 80.9\\
JAN\cite{long2017jan} & 86.0 & 70.7 & 85.1 & 69.2 & 99.7 & 96.7 & 84.6\\
MADA\cite{pei2018multi} & 90.0 & 66.4 & 87.8 & 70.3 & 99.6 & 97.4 & 85.2 \\
SimNet\cite{pinheiro2018unsupervised}& 88.6 & 71.8 & 85.3 & 73.4 & 99.7 & 98.2 & 86.2\\
GTA\cite{sankaranarayanan2017generate}& 89.5 & 71.4 & 87.7 & 72.8 & 99.8 & 97.9 & 86.5\\ 
DAAA\cite{kang2018deep}& 86.8 & 73.9 & 88.8 & 74.3 & \textbf{100.0} & \textbf{99.3} & 87.2\\
% iCAN\cite{zhang2018collaborative} & 92.5 & 69.9 & 90.1 & 72.1 & \textbf{100.0} & 98.8 & 87.2\\
CAN(Ours) & 81.5 & 63.4 & 85.5 & 65.9 & 99.7 & 98.2 & 82.4 \\
% CDAN+E\cite{long2018conditional} & \textbf{94.1} & 69.3& \textbf{92.9} & 71.0 & \textbf{100.0} & 98.6 & 87.7\\
SPCAN(Ours) & \textbf{92.4} & \textbf{74.5} & \textbf{91.2} & \textbf{77.1} & \textbf{100.0} & 99.2 & \textbf{89.1} \\ % 92.6 & \textbf{74.4} & 91.1 & \textbf{76.5} & \textbf{100.0} & 99.1 &  \textbf{89.0}
\bottomrule
\end{tabular}
\end{center}
% \vspace{-2mm}
\end{table}

\begin{table}
\captionsetup{justification=centering, labelsep=newline, font=sf, textfont=footnotesize}
\setlength\tabcolsep{3.5pt}\small
\begin{center}
\caption{Accuracies (\%) of different methods on the ImageCLEF-DA dataset.}
\label{table:clefmain}
\vspace{-1mm}
\begin{tabular}{c c c c c c c c}
\toprule
\textbf{Methods} & I$\rightarrow$P & P$\rightarrow$I & I$\rightarrow$C & C$\rightarrow$I & C$\rightarrow$P & P$\rightarrow$C & Avg.\\
\midrule
ResNet50\cite{he2016deep} & 74.6 & 82.9 & 91.2 & 79.8 & 66.8 & 86.9 & 80.4 \\
DAN\cite{long2015learning} & 74.5 & 82.2 & 92.8 & 86.3 & 69.2 & 89.8 & 82.5\\
RTN\cite{long2016unsupervised} & 74.6 & 85.8 & 94.3 & 85.9 & 71.7 & 91.2 & 83.9\\
DANN\cite{ganin2016domain} & 75.6 & 84.0 & 93.0 & 86.0 & 71.7 & 87.5 & 83.0\\
JAN\cite{long2017jan} & 76.8 & 88.0 & 94.7 & 89.7 & 74.2 & 91.7 & 85.8\\
MADA\cite{pei2018multi} & 75.0 & 87.9 & \textbf{96.0} & 88.8 & 75.2 & \textbf{92.2} & 85.8 \\
CAN(Ours) & 78.2 & 87.5 & 94.2 & 89.5 & 75.8 & 89.2 & 85.7 \\
% iCAN\cite{zhang2018collaborative} & \textbf{79.5} & 89.7 & 94.7 & 89.9 & 78.5 & 92.0 & 87.4\\
% CDAN+E\cite{long2018conditional} & 77.7 & 90.7 & \textbf{97.7} & 91.3 & 74.2 & 94.3 & 87.7\\
SPCAN(Ours) & \textbf{79.0} & \textbf{91.1} & 95.5 & \textbf{92.9} & \textbf{79.4 }& 91.3 & \textbf{88.2 }\\ %78.7 & \textbf{91.2} & 95.3 &\textbf{ 92.7} & \textbf{79.6} & 91.5 & \textbf{88.2}
\bottomrule
\end{tabular}
\end{center}
\vspace{-1mm}
\end{table}

\begin{table}[t]
\captionsetup{justification=centering, labelsep=newline, font=sf, textfont=footnotesize}
\setlength\tabcolsep{3.5pt}\small
\begin{center}
\caption{Average Accuracies (\%) of different methods on the VISDA-2017 dataset. All methods use ResNet50 as the backbone network except ADR\cite{saito2017adversarial}, which uses ResNet101.} 
%* For our methods CAN and SPCAN, we only use a small subset of the VISDA-2017 training dataset as the source training data.
\label{table:visdamain}
\vspace{-1mm}
\begin{tabular}{c c c c c c c c}
\toprule
\textbf{Methods} & VISDA-2017\\
\midrule
ResNet50\cite{he2016deep} & 50.6 \\
DAN\cite{long2015learning} & 55.0 \\
RTN\cite{long2016unsupervised} & 57.3 \\
DANN\cite{ganin2016domain} & 57.8 \\
JAN\cite{long2017jan} & 61.8 \\
GTA\cite{sankaranarayanan2017generate} & 69.5 \\
SimNet\cite{pinheiro2018unsupervised} & 69.6 \\
% CDAN+E\cite{long2018conditional} & 70.0 \\
ADR\cite{saito2017adversarial} & 73.5 \\
% \midrule
% DANN-s \cite{ganin2016domain} & 57.6 \\
CAN(Ours) & 64.1\\
% iCAN-s(Ours) & 76.9\\
SPCAN(Ours) & \textbf{79.4}\\
\bottomrule
\end{tabular}
\end{center}
% \vspace{-5mm}
\end{table}

The results on three datasets, Office-31, ImageCLEF-DA and VISDA-2017 are reported in Table~\ref{table:officemain}, Table~\ref{table:clefmain} and Table~\ref{table:visdamain}, respectively. In terms of the average accuracy, our proposed SPCAN achieves the best results on all three datasets, which demonstrates the effectiveness of our newly proposed SPCAN method for domain adaptation.  When compared with other state-of-the-art approaches on the Office-31 dataset, our method SPCAN achieves prominent improvements from the two more difficult tasks, (\ie, W$\rightarrow$A and D$\rightarrow$A). % For VISDA-2017, we observe that the result of our method SPCAN using a small proportion of training data (\ie, under VISDA Setting 2) is already sufficient to demonstrate the effectiveness of our SPCAN, which outperforms other state-of-the-art methods (\eg, \cite{long2017jan,sankaranarayanan2017generate,pinheiro2018unsupervised,saito2017adversarial}) by using the whole training set (\ie, under VISDA Setting 1).
% SPCAN also outperforms iCAN in our preliminary work \cite{zhang2018collaborative} on all three datasets, in which the average accuracies are 87.2\% on Office-31, 87.4\% on Image-CLEF and 76.9\% on VISDA-2017, respectively.

%There are 152,397 synthetic images from 12 categories in the VISDA-2017 dataset. However, a few images from the same 3D model and similar viewpoints are redundant for model training. As a result, we randomly use 1000 images from each class to form a small subset of VISDA-2017, which are used as the source training data to evaluate our methods CAN and SPCAN. We observe that the results of our methods CAN and SPCAN using a small proportion of training data are already sufficient to demonstrate the effectiveness of our SPCAN, which also outperforms other state-of-the-art methods (\eg, \cite{long2017jan,sankaranarayanan2017generate,pinheiro2018unsupervised,saito2017adversarial}) by using the whole training set.

\vspace{-1mm}
\subsubsection{Analysis on the learnt weights $\lambda_l$'s}

In our method CAN, we follow the existing method DANN \cite{ganin2016domain} and set the weight at the last layer/block as a negative value to learn domain invariant representations at higher layers/blocks. In order to empirically justify why it is useful to learn domain specific representations at lower layers/blocks, we also conduct new experiments by using different fixed weights (\eg, $\lambda_l=\frac{1}{3}$, $\lambda_l=-\frac{1}{3}$, $\lambda_l=1$, $\lambda_l=-1$, $l=1,2,3$) instead of learning the optimal weights with our method SPCAN, which is referred to as sSPCAN($\frac{1}{3}$), sSPCAN($-\frac{1}{3}$), sSPCAN($1$), and sSPCAN($-1$), respectively. The results are reported in Table~\ref{table:caanalysis}. From the results in Table~\ref{table:caanalysis}, we observe that our SPCAN outperforms sSPCAN for all tasks, which demonstrates that it is beneficial to learn the optimal weights for different blocks. For sSPCAN, we also observe that sSPCAN achieves the best results when setting $\lambda_l=\frac{1}{3},l=1,2,3$, which also indicates that it is beneficial to learn domain specific features in the shallower layers. 

\begin{table}[t]
\captionsetup{justification=centering, labelsep=newline, font=sf, textfont=footnotesize}
\setlength\tabcolsep{1.5pt}\small
\begin{center}
\caption{Accuracies (\%) of the simplified version of SPCAN (referred to as sSPCAN) by using different sets of fixed $\lambda_l$'s on the Office-31 dataset.}
\label{table:caanalysis}
% \vspace{-1mm}
\resizebox{\linewidth}{!}{%
\begin{tabular}{c c c c c c c c}
\toprule
\textbf{Methods} & A$\rightarrow$W & W$\rightarrow$A & A$\rightarrow$D & D$\rightarrow$A & W$\rightarrow$D & D$\rightarrow$W & Avg.\\
\midrule
sSPCAN$(\frac{1}{3})$ & 89.5 & 68.6 & 91.1 & 74.3 & 100.0 & 98.9 & 87.1 \\
sSPCAN$(-\frac{1}{3})$  & 89.4 & 68.8 & 90.3 & 68.8 & 99.8 & 98.4 & 85.9\\
sSPCAN$(1)$ & 89.2 & 72.2 & 87.8 & 70.4 & 99.8 & 98.3 & 86.3\\
sSPCAN$(-1)$  & 88.7 & 70.8 & 88.5 & 71.0 & 99.6 & 98.0 & 86.1\\
\midrule
SPCAN & \textbf{92.4} & \textbf{74.5} & \textbf{91.2} & \textbf{77.1} & \textbf{100.0} & \textbf{99.2} & \textbf{89.1} \\ %
\bottomrule
\end{tabular}
}
\end{center}
\end{table}

% after comparison .. automatic learn 

% we further analysis the trend of learnt lambda (bar image)

% similar observeration

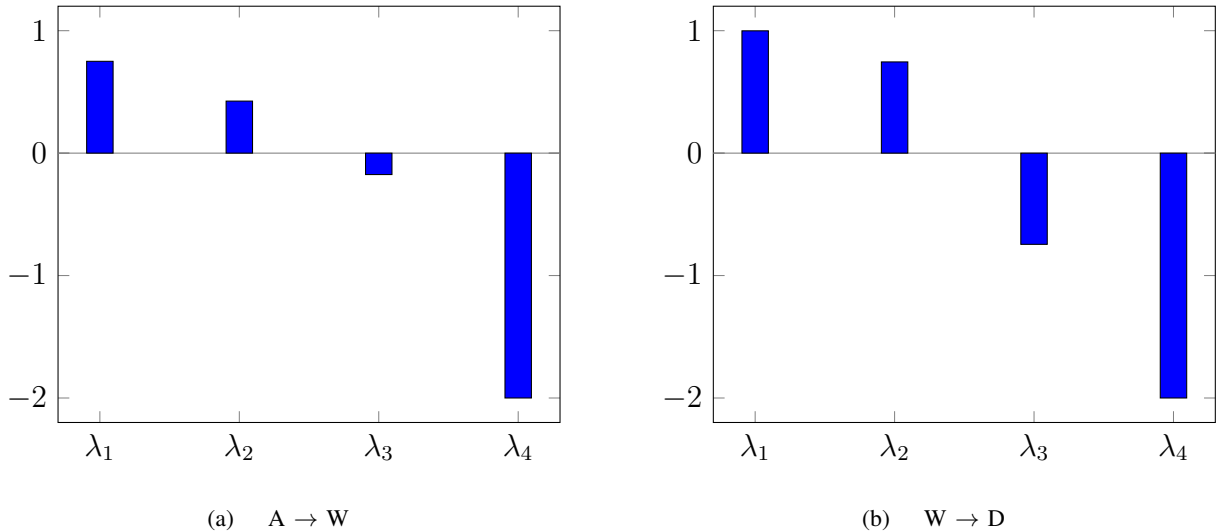
\begin{figure}[H]
    % \vspace{-2mm}
    \begin{center}
    \captionsetup{justification=centering}
    \subfigure[\quad A $\rightarrow$ W]{
    \begin{tikzpicture}
        \begin{axis}[
            width  = 0.5\linewidth,
            symbolic x coords={$\lambda_1$, $\lambda_2$, $\lambda_3$, $\lambda_4$},
            ymax=1.2,
            ymin=-2.2,
            xtick=data, extra y ticks = 0,
            extra y tick labels={},
            scaled y ticks = false,
            extra y tick style={grid=major,major grid style={thin,draw=gray}}]
            \addplot[ybar,fill=blue] coordinates {
                    ($\lambda_1$,   0.75)
                    ($\lambda_2$,  0.425)
                    ($\lambda_3$,   -0.175)
                    ($\lambda_4$,   -2)
            };
        \end{axis}
    \end{tikzpicture}
    }
    \hfill
    \subfigure[\quad W $\rightarrow$ D]{
    \begin{tikzpicture}
        \begin{axis}[
            width  = 0.5*\linewidth,
            symbolic x coords={$\lambda_1$, $\lambda_2$, $\lambda_3$, $\lambda_4$},
            ymax=1.2,
            ymin=-2.2,
            xtick=data, extra y ticks = 0,
            extra y tick labels={},
            scaled y ticks = false,
            extra y tick style={grid=major,major grid style={thin,draw=gray}}]
            \addplot[ybar,fill=blue] coordinates {
                    ($\lambda_1$,   0.999)
                    ($\lambda_2$,  0.745)
                    ($\lambda_3$,   -0.745)
                    ($\lambda_4$,   -2)
            };
        \end{axis}
    \end{tikzpicture}
    }
    % \vspace{-3mm}
    \caption{Different $\lambda_l$'s learnt by using our SPCAN on the Office-31 dataset (a) A$\rightarrow$W and (b) W$\rightarrow$D.}
    \label{fig:canbar}
    \end{center}
% \vspace{-2mm}
\end{figure}

We also observe that the learnt $\lambda_l$'s at lower blocks (\eg, $\lambda_1$) are often larger than those at higher blocks (\eg, $\lambda_3$). We take the Office-31 dataset A$\rightarrow$W and W$\rightarrow$D as two examples to illustrate $\lambda_l$'s learnt at different blocks in Figure \ref{fig:canbar}. The results show that the learnt representations can be gradually changed from domain specific representations at lower blocks to domain invariant representations at higher blocks. 

\begin{table}[t]
\captionsetup{justification=centering, labelsep=newline, font=sf, textfont=footnotesize}
\setlength\tabcolsep{1.5pt}\small
\begin{center}
\caption{Accuracy (\%) comparison between our SPCAN and the baseline method DANN when using different CNN architectures as the backbone networks on the Office-31 dataset.}
\label{table:extractor}
% \vspace{-1mm}
\resizebox{\linewidth}{!}{%
\begin{tabular}{c c c c c c c c}
\toprule
\textbf{Methods} & A$\rightarrow$W & W$\rightarrow$A & A$\rightarrow$D & D$\rightarrow$A & W$\rightarrow$D & D$\rightarrow$W & Avg.\\
\midrule
DANN(VGG19) & 75.3 & 60.7 & 77.6 & 60.1 & 99.3 & 97.6 & 78.4\\
DANN(ResNet50)  & 79.3 & 63.2 & 80.7 & 65.3 & 99.6 & 97.3 & 80.9\\
DANN(ResNet152) & 79.8 & 63.9 & 82.6 & 65.9 & 99.7 & 97.4 & 81.6\\
DANN(DenseNet161)  & 83.1 & 65.1 & 84.7 & 66.0 & 99.6 & 97.6 & 82.7\\
\midrule
SPCAN(VGG19) & 90.3 & 72.5 & 89.0 & 74.9 & 100.0 & 98.2 & 87.5 \\
SPCAN(ResNet50) & 92.4 & 74.5 & 91.2 & 77.1 & 100.0 & 99.2 & 89.1 \\
SPCAN(ResNet152) & 92.8 & 77.3 & 92.0 & 78.6 & 100.0 & 99.3 & 90.0\\
SPCAN(DenseNet161) & 93.8 & 76.8 & 93.6 & 76.9 & 100.0 & 99.3 & 90.1 \\
\bottomrule
\end{tabular}
}
\end{center}
\end{table}

\vspace{-1mm}
\subsubsection{Results using different CNN architectures}

In our proposed methods CAN and SPCAN, we use ResNet50 \cite{he2016deep} as our default backbone network. However, different CNN architectures \cite{he2016deep}, \cite{Simonyan14c}, \cite{huang2017densely} can be readily used as the backbone networks in our work. In Table~\ref{table:extractor}, we take the Office-31 dataset as an example to report more results of our method SPCAN when using VGG19 \cite{Simonyan14c}, ResNet152 \cite{he2016deep} and DenseNet161 \cite{huang2017densely} as the backbone networks. Similar to our default model using ResNet50 as the backbone network, we add four domain discriminators after the last pooling layers of the four blocks when using different CNN architectures. For comparison, the results of the baseline method DANN \cite{ganin2016domain} based on VGG19 \cite{Simonyan14c}, ResNet152 \cite{he2016deep} and DenseNet161 \cite{huang2017densely} are also reported in Table~\ref{table:extractor}. In terms of the average accuracies over all six settings on the Office-31 dataset, we observe that SPCAN outperforms DANN when using VGG19, ResNet152 and DenseNet161 as the backbone networks, which demonstrates the effectiveness and the robustness of our method SPCAN for unsupervised domain adaptation. 

\vspace{-1mm}
\begin{table}[!t]
\captionsetup{justification=centering, labelsep=newline, font=sf, textfont=footnotesize}
\setlength\tabcolsep{1pt}\small
\begin{center}
\caption{Ablation study of different variants of our SPCAN method on the Office-31 dataset.}
\label{table:moduleanalysis}
\vspace{-1mm}
\resizebox{\linewidth}{!}{%
\begin{tabular}{c c c c c c c c}
\toprule
\textbf{Methods} & A$\rightarrow$W & W$\rightarrow$A & A$\rightarrow$D & D$\rightarrow$A & W$\rightarrow$D & D$\rightarrow$W & Avg.\\
\midrule
%DANN & 79.3 & 63.2 & 80.7 & 65.3 & 99.6 & 97.3 & 80.9\\
CAN & 81.5 & 63.4 & 85.5 & 65.9 & 99.7 & 98.2 & 82.4 \\
% DANN+TSS(I) & 80.0 & 65.9 & 81.9 & 66.5 & 99.7 & 97.8 & 82.0\\
% CAN+TSS(I) & 86.6 & 67.7 & 86.8 & 69.9 & 99.8 & 98.3 & 84.9\\
% DANN+TSS(I+D)  & 85.3 & 69.6 & 86.1 & 70.7 & 99.8 & 98.3 & 85.0\\
% iCAN & 92.5 & 69.9 & 90.1 & 72.1 & 100.0 & 98.8 & 87.2\\
% T1DF+T2D,T3F & 90.3 & 73.0 & 89.3 & 75.0 & & \\
% T1DF+T3D,T2F & 91.5 & 74.2 & 88.4 & 74.1 & & & \\
% N/A & 91.9 & 73.8 & 90.5 & 76.2 & 99.9 & 98.9 & 88.5 \\
% TC & 92.0 & 74.1 & 90.6 & 76.0 & 100.0 & 99.0 & 88.6 \\
% TD1 & 91.7 & 73.0 & 88.8 & 75.6 & 99.9 & 98.9 & 88.0 \\
% TD2 & 91.6 & 72.5 & 89.3 & 73.5 & 99.8 & 98.7 & 87.6 \\
% TC$\cup$TD1 & 92.6 & 74.4 & 91.1 & 76.5 & 100.0 & 99.1 & 89.0\\
% TC$\cup$TD2 & 90.7 & 74.0 & 91.0 & 74.7 & 99.9 & 99.0 & 88.2\\
% DANN+SSC($\mathbb{S}^c$) \\
SPCAN w/o CSS & 81.8 & 63.6 & 84.9 & 66.0 & 99.8 & 98.0 & 82.4 \\
SPCAN w/o DSS & 90.3 & 73.8 & 90.2 & 75.4 & 99.9 & 98.9 & 88.1\\% 92.0 & 74.1 & 90.6 & 76.0 & 100.0 & 99.0 & 88.6 
% DANN+SSD($\mathbb{S}^d$) \\
% CAN+SSC($\mathbb{S}^c$)+SSD($\mathbb{S}^c$)\\
% SPCAN w/o SSD($\mathbb{S}^c$) & 89.8 & 74.4 & 90.0 & 76.8 &  100.0 & 99.0 & 88.3\\
% CAN+SSC($\mathbb{S}^c$)+SSD($\mathbb{S}^c+\mathbb{S}^d$)\\
SPCAN & 92.4 & 74.5 & 91.2 & 77.1 & 100.0 & 99.2 & 89.1 \\
\midrule
DANN & 79.3 & 63.2 & 80.7 & 65.3 & 99.6 & 97.3 & 80.9\\
SPDANN & 88.6 & 70.1 & 89.2 & 71.3 & 99.8 & 98.3 & 86.2\\
\bottomrule
\end{tabular}
}
\end{center}
\vspace{-2mm}
\end{table}

\vspace{-1mm}
\begin{figure*}[!t]
% \fbox{\rule{0pt}{2in} \rule{.9\linewidth}{0pt}}
\includegraphics[width=\linewidth]{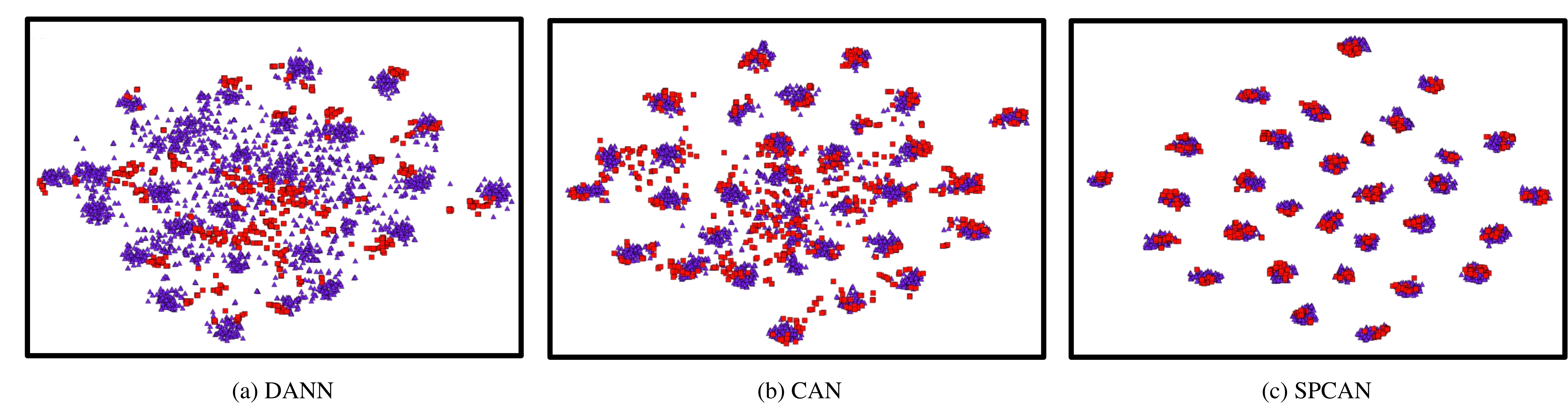}
\caption{Visualization of the source and target domain samples in the 2D space for Office-31 (A$\rightarrow$W)  by using the t-SNE embedding method \cite{maaten2008visualizing}, where the representations are learnt by using the baseline method DANN (a), our method CAN (b) and SPCAN (c), respectively. The samples from the source and the target domains are shown in blue and red colors, respectively. Best viewed in color.}
\vspace{-1mm}
\label{fig:adaptvis}
\end{figure*}

\vspace{-1mm}
\subsubsection{Ablation study}
\label{sec:modelanalysis}
% \wenli{The goal here is to validate our claims on two selection modules. Maybe you could move Table 4 in the conference paper back, and conduct the similar experiments for SPCAN. It is interesting to show 1) the sample selection form image classifier generally works for DA methods, 2) the new sample selection strategy in SPCAN is more effective than iCAN.}
% To further investigate the different design variants, especially with a focus on the importance geometric cue in the two components. We conduct further ablation studies on the two adaptation modules individually in below.
To investigate the benefits of different modules in SPCAN, in this section, we conduct ablation study on SPCAN using the Office-31 dataset. In particular, we study different variants of SPCAN, including:
(1) CAN corresponds to the proposed CAN model without using selected target samples for re-training the image classifier or the domain classifiers. 
(2) SPCAN w/o CSS is the SPCAN method without using the selected target samples for re-training the image classifier.
(3) SPCAN w/o DSS is the SPCAN method without using the selected target samples for re-training the domain classifiers.
% (4) SPCAN w/o SSD($\mathbb{S}^c$) is the (SPCAN) method that uses set $\mathbb{S}^c$ to re-train the image classifier but only uses set $\mathbb{S}^d$(without using set $\mathbb{S}^c$) to re-train the domain classifiers.
% (5) SPCAN is the model with all the components of this work.
(4) SPCAN is our final model, in which both the image classifier and the domain classifier are re-trained by using the selected target samples.
% Our sample selection strategy uses set $\mathbb{S}^c$ to re-train the image classifier and uses both set $\mathbb{S}^c$ and set $\mathbb{S}^d$ as in \ref{eqn:target_d} to re-train the domain classifiers.
(5) DANN corresponds to our re-implemented DANN model in \cite{ganin2016domain}, which does not use the selected target samples for re-training the image classifier or the domain classifiers. 
(6) SPDANN applies the same sample selection strategy as in our SPCAN model, but the selected samples are used to re-train DANN.
% , instead of our proposed CAN model.

% group is for improving the image classifier, which is referred to as SSC, and another is for improving the domain classifiers. 
% We will perform an ablation study to discuss the effectiveness of different selection modules.
% As shown in Table~\ref{table:spldomcompare}, it is able to find that TC set is important for domain classifier as well and the union of TC and TD1 set achieves the best performance, which indicates the best selection strategy we use in our SPCAN model.
% How about still separating the two selection analysis? Or SPCAN-(

% \textcolor{blue}{
The results of all variants are reported in Table~\ref{table:moduleanalysis}.
% from the average accuracies point of view, we have following observations. 
Our CAN model without using any sample selection modules achieves the average accuracy of $82.4\%$. 
Moreover, we also observe that the results of SPCAN w/o CSS achieves only similar results with the CAN model. On the other hand, if we use only the CSS module in SPCAN (\ie, SPCAN w/o DSS), the average accuracy can be improved by 5.7\% to 88.1\%.  The above results indicate that it is crucial to use image classifier to select pseudo-labeled target samples.
Finally, our SPCAN module achieves the average accuracy of 89.1\%, which demonstrates that it is beneficial to integrate the new self-paced sample selection module for improving the domain adaptation performance.

Furthermore, we also show that our self-paced sample selection module can also be applied to other adversarial training based methods to further improve their performance for domain adaptation. For example, when combining our self-paced sample selection module with DANN, the average accuracy is improved from $80.9\%$ to $86.2\%$ (see the last two rows in Table~\ref{table:moduleanalysis}). However, it is still worse than our SPCAN, which also validates the effectiveness of our collaborative and adversarial training scheme. 

\vspace{-1mm}
\subsubsection{Qualitative Analysis}

We take the Office-31 dataset (A $\rightarrow$ W) as an example to visualize the features extracted by using the baseline method DANN \cite{ganin2015unsupervised, ganin2016domain}, our methods CAN and SPCAN. For better visualization, we follow the existing domain adaptation methods \cite{long2015learning,long2017jan} to generate the 2D representations of all source and target domain samples from all 31 classes. The visualization results are shown in Figure \ref{fig:adaptvis}. When comparing our methods CAN and SPCAN with the baseline method DANN, we observe that the distributions of both source data and target data have more overlaps than DANN, and the samples from different classes are also better grouped by using CAN \cite{zhang2018collaborative} and SPCAN. The results clearly demonstrate that the features extracted by using our methods CAN \cite{zhang2018collaborative} and SPCAN are not only domain-invariant but also discriminant for the image classification task. Compared with CAN, SPCAN is able to learn more discriminant and domain-invariant features.

% We also take the Office-31 dataset (A $\rightarrow$ W) as an example to visualize the representations of all samples from both domains at the 10th layer, the 22th layer, the 40th layer and the 49th layer. As shown in Figure \ref{fig:diffblock}, the two distributions of source domain data and the target domain data have less overlaps when using the representations from higher layers, while the two distributions have more overlaps when using the representations from lower layers. These results demonstrate that our method iCAN can learn domain informative representations at lower layers and domain uninformative representations at higher layers.

\vspace{-1mm}
\subsection{Experiments for the video action recognition task}

% \vspace{-1mm}
\subsubsection{Datasets}
%For the video action recognition task, there are hardly any benchmark action recognition dataset with multiple distinct video subsets that share the same action class. 
Generally, videos from different datasets have different data distributions. As a result, we evaluate our methods % using a shared category 
on two commonly used action recognition datasets, UCF101\cite{soomro2012ucf101} and HMDB51\cite{kuehne2011hmdb}, which contain 10 common categories from both datasets.
 
The UCF101 dataset \cite{soomro2012ucf101} is a benchmark dataset for action recognition. It contains 101 action classes of 13,320 video clips. The HMDB51 dataset \cite{kuehne2011hmdb} is a large video collection from web videos and movies. It consists of 6766 annotated videos from 51 different actions. From these two datasets, we select all video clips from 10 common categories\footnote{The ten common categories are Archery(Shoot Bow), Basketball(Shoot Ball), Biking(Ride Bike), Diving, Fencing, Golf Swing(Golf), Horse Riding(Ride Horse), Pull Ups, Punch, and Push Ups.} for domain adaptation. The two subsets are referred to as UCF101-10 and HMDB51-10 in this work, which contains 1339 videos and 1172 videos, respectively. 
% , which we call the constructed dataset UCF101-HMDB51-10. 
% We call UCF101-10(U) and HMDB51-10(H) the two subsets of the selected  to perform domain adaptation. UCF101-10 contains 1339 videos and HMDB51-10 contains 1172 videos.

\vspace{-1mm}
\subsubsection{Implementation Details}
In this task, we also implement our methods in the Pytorch framework. The BN-Inception \cite{ioffe2015batch} model pre-trained based on the ImageNet dataset is used as the backbone network. 
% \wenli{Please give a reference to BN-Inception}.
The learning rate for the domain discriminators and the video classifier is set as ten times of the backbone network. We also add four domain discriminators for the BN-Inception model, after the 2$^{nd}$, 3$^{rd}$, 4$^{th}$ and 5$^{th}$ pooling layer. The video classifier is added after the 5$^{th}$ pooling layer. Stochastic gradient descent (SGD) is also used for optimization. 

We set the initial learning rate as 0.001 for the TS-SPCAN model and gradually decrease the learning rate after each iteration as in SPCAN. We set the batch size as 48 for both streams and utilize the same setting for other hyper-parameters as in our SPCAN method. 

\vspace{-1mm}
\subsubsection{Experimental Results}
The results for the video action recognition task under the domain adaptation setting between the UCF101-10 and HMDB51-10 datasets are reported in Table~\ref{table:ucf101-hmbd51-10}. We compare our proposed approach with the state-of-the-art methods, including the TSN method \cite{wang2016temporal}, which is fine-tuned from the model pre-trained based on ImageNet dataset, and our re-implemented DANN\cite{ganin2016domain}. 
The TSN network is trained by using the labeled samples from the source domain only without considering domain shift, while for DANN we further add a domain discriminator after the last pooling layer of TSN for each stream. To validate the effectiveness of our pseudo label exchanging strategy, we also include our CAN and SPCAN models for comparison. For the aforementioned methods, the prediction scores from the two streams are fused with the same weights as in TSN\cite{wang2016temporal}. 

% The TSN network is trained using source domain labeled samples only without considering domain shift, while for DANN we further add a domain discriminator on top of TSN. For all baseline methods, the BN-Inception network is used as the backbone network. Two networks are trained individually on the RGB stream and optical flow stream, and the their predictions are fused with the same weights as in TSN\cite{wang2016temporal}. 

The results are summarized in Table~\ref{table:ucf101-hmbd51-10}. We observe that all domain adaption approaches improve the classification accuracy when compared with the baseline method TSN, and our TS-SPCAN achieves the best performance. In terms of the average video action recognition accuracy, DANN improves TSN by $2.4\%$ by using the domain adversarial learning strategy, while our CAN method further improves the results by $2.7\%$, which validates the effectiveness of our collaborative and adversarial learning scheme. SPCAN outperforms CAN, which shows it is beneficial to select the pseudo-labeled target samples by using our newly proposed modules. Finally, our TS-SPCAN achieves the best average accuracy of $93.5\%$, which demonstrates the effectiveness of our TS-SPCAN by jointly integrating the complementary information from different streams to learn robust models for cross-domain video action recognition.

% We proposed two variations for our proposed method, including TS-SPCAN using selection criteria 1 and CoSPCAN-Fu using selection criteria 2. We compare the results with the state-of-the-art methods including the ImageNet finetuned BN-Inception\cite{ioffe2015batch} network, our re-implemented DANN\cite{ganin2016domain}, and our work CAN. All methods utilize the BN-Inception network as the base architecture, and use the RGB and Optical flow  stream fusion with the same weights as in TSN\cite{wang2018temporal}.

% From the result, it is able to find that from the average accuracy of the fused prediction point of view, our newly proposed SPCAN-A gains 0.6\% improvements compare to our preliminary work iCAN. The two extended methods CoSPCAN and CoSPCAN-Fu are able to provide even better domain adaptation results, which improved 1.3\% and 1.6\%, respectively.

\begin{table}
\captionsetup{justification=centering, labelsep=newline, font=sf, textfont=footnotesize}
\setlength\tabcolsep{1.7pt}\small
\begin{center}
\caption{Comparison of different methods for video action recognition on the UCF101-10 (U) and HMDB51-10 (H) datasets, in which BN-Inception \cite{ioffe2015batch} is used as the backbone network. }
\label{table:ucf101-hmbd51-10}
% \vspace{-1mm}
\begin{tabular}{c c c c c c c c}
\toprule
\textbf{Method} & U$\rightarrow$H & H$\rightarrow$U & Avg.\\
\midrule
TSN\cite{wang2016temporal} & 80.5 & 90.1 & 85.3\\
DANN\cite{ganin2016domain} & 83.6 & 91.7 & 87.7 \\
\midrule
CAN(Ours)  & 85.7 & 95.1 & 90.4 \\
% iCAN\cite{zhang2018collaborative} & 86.8 & 96.1 & 91.5 \\
SPCAN(Ours) & 87.6 & 96.6 & 92.1 \\
TS-SPCAN(Ours) & \textbf{89.4} & \textbf{97.6} & \textbf{93.5} \\
% CoSPCAN-Fu(Ours) & 89.0 & 98.3 & 93.7 \\
% CoSPCAN-ExAdd(Ours) & 89.6 & 97.7 & 93.7 \\
% CoSPCAN-FuAdd(Ours) & 90.3 & 98.8 & 94.6 \\
\bottomrule
\end{tabular}
\end{center}
\vspace{-1mm}
\end{table}

\vspace{-2mm}
\section{Conclusion}
In this paper, we have proposed a new deep learning method called Collaborative and Adversarial Network
(CAN) for unsupervised domain adaptation. Different from
the existing works that learn only domain invariant
representations through domain adversarial learning, our
method CAN additionally learns domain specific representations through domain collaborative learning. To effectively exploit the unlabeled target samples, we have further proposed a Self-Paced CAN (SPCAN) method, in which a self-paced learning scheme is developed to iteratively select pseudo-labeled target samples, and enlarge the training set for learning a more robust model. Observing that the two-stream framework is commonly used for video action recognition, we have also developed the unsupervised domain adaptation method Two-stream SPCAN (TS-SPCAN) for video action recognition, in which we further integrate the complementary information from different views (\ie, RGB and Flow clues) by using the selected pseudo-labeled samples from one view to help the model training process on the other view. Extensive experiments on multiple benchmark datasets have demonstrated the effectiveness of our proposed methods.

% use section* for acknowledgment
% \ifCLASSOPTIONcompsoc
%   % The Computer Society usually uses the plural form
%   \section*{Acknowledgments}
%   This work was supported by the Australian Research Council Future Fellowship under Grant FT180100116.
% \else
  % regular IEEE prefers the singular form
\vspace{-2mm}
\section*{Acknowledgment}
This work was supported by the Australian Research Council (ARC) Future Fellowship under Grant FT180100116 and ARC DP200103223.

% The authors would like to thank...

% Can use something like this to put references on a page
% by themselves when using endfloat and the captionsoff option.
\ifCLASSOPTIONcaptionsoff
  \newpage
\fi
\vspace{-2mm}
{\small
\bibliographystyle{IEEEtran}
\bibliography{main}

% Generated by IEEEtran.bst, version: 1.14 (2015/08/26)
\begin{thebibliography}{10}
\providecommand{\url}[1]{#1}
\csname url@samestyle\endcsname
\providecommand{\newblock}{\relax}
\providecommand{\bibinfo}[2]{#2}
\providecommand{\BIBentrySTDinterwordspacing}{\spaceskip=0pt\relax}
\providecommand{\BIBentryALTinterwordstretchfactor}{4}
\providecommand{\BIBentryALTinterwordspacing}{\spaceskip=\fontdimen2\font plus
\BIBentryALTinterwordstretchfactor\fontdimen3\font minus
  \fontdimen4\font\relax}
\providecommand{\BIBforeignlanguage}[2]{{%
\expandafter\ifx\csname l@#1\endcsname\relax
\typeout{** WARNING: IEEEtran.bst: No hyphenation pattern has been}%
\typeout{** loaded for the language `#1'. Using the pattern for}%
\typeout{** the default language instead.}%
\else
\language=\csname l@#1\endcsname
\fi
#2}}
\providecommand{\BIBdecl}{\relax}
\BIBdecl

\bibitem{duan2012domainkernel}
L.~Duan, I.~W. Tsang, and D.~Xu, ``Domain transfer multiple kernel learning,''
  \emph{T-PAMI}, vol.~34, no.~3, pp. 465--479, 2012.

\bibitem{baktashmotlagh2013unsupervised}
M.~Baktashmotlagh, M.~T. Harandi, B.~C. Lovell, and M.~Salzmann, ``Unsupervised
  domain adaptation by domain invariant projection,'' in \emph{ICCV}, 2013, pp.
  769--776.

\bibitem{duan2012domain}
L.~Duan, D.~Xu, and I.~W.-H. Tsang, ``Domain adaptation from multiple sources:
  A domain-dependent regularization approach,'' \emph{T-NNLS}, vol.~23, no.~3,
  pp. 504--518, 2012.

\bibitem{fernando2013unsupervised}
B.~Fernando, A.~Habrard, M.~Sebban, and T.~Tuytelaars, ``Unsupervised visual
  domain adaptation using subspace alignment,'' in \emph{CVPR}, 2013, pp.
  2960--2967.

\bibitem{duan2012visual}
L.~Duan, D.~Xu, I.~W.-H. Tsang, and J.~Luo, ``Visual event recognition in
  videos by learning from web data,'' \emph{T-PAMI}, vol.~34, no.~9, pp.
  1667--1680, 2012.

\bibitem{li2014learning}
W.~Li, L.~Duan, D.~Xu, and I.~W. Tsang, ``Learning with augmented features for
  supervised and semi-supervised heterogeneous domain adaptation,''
  \emph{T-PAMI}, vol.~36, no.~6, pp. 1134--1148, 2014.

\bibitem{li2018visual}
W.~Li, L.~Chen, D.~Xu, and L.~Van~Gool, ``Visual recognition in rgb images and
  videos by learning from rgb-d data,'' \emph{T-PAMI}, vol.~40, no.~8, pp.
  2030--2036, 2018.

\bibitem{szegedy2015going}
C.~Szegedy, W.~Liu, Y.~Jia, P.~Sermanet, S.~Reed, D.~Anguelov, D.~Erhan,
  V.~Vanhoucke, and A.~Rabinovich, ``Going deeper with convolutions,'' in
  \emph{CVPR}, 2015, pp. 1--9.

\bibitem{he2016deep}
K.~He, X.~Zhang, S.~Ren, and J.~Sun, ``Deep residual learning for image
  recognition,'' in \emph{CVPR}, 2016, pp. 770--778.

\bibitem{long2015fully}
J.~Long, E.~Shelhamer, and T.~Darrell, ``Fully convolutional networks for
  semantic segmentation,'' in \emph{CVPR}, 2015, pp. 3431--3440.

\bibitem{simonyan2014two}
K.~Simonyan and A.~Zisserman, ``Two-stream convolutional networks for action
  recognition in videos,'' in \emph{NeurIPS}, 2014, pp. 568--576.

\bibitem{he2017mask}
K.~He, G.~Gkioxari, P.~Doll{\'a}r, and R.~Girshick, ``Mask r-cnn,'' in
  \emph{ICCV}, 2017, pp. 2961--2969.

\bibitem{hu2015deep}
J.~Hu, J.~Lu, and Y.-P. Tan, ``Deep transfer metric learning,'' in \emph{CVPR},
  2015, pp. 325--333.

\bibitem{long2015learning}
M.~Long, Y.~Cao, J.~Wang, and M.~Jordan, ``Learning transferable features with
  deep adaptation networks,'' in \emph{ICML}, 2015, pp. 97--105.

\bibitem{long2016unsupervised}
M.~Long, H.~Zhu, J.~Wang, and M.~I. Jordan, ``Unsupervised domain adaptation
  with residual transfer networks,'' in \emph{NeurIPS}, 2016, pp. 136--144.

\bibitem{long2017jan}
------, ``Deep transfer learning with joint adaptation networks,'' in
  \emph{ICML}, 2017, pp. 2208--2217.

\bibitem{ganin2015unsupervised}
Y.~Ganin and V.~Lempitsky, ``Unsupervised domain adaptation by
  backpropagation,'' in \emph{ICML}, 2015, pp. 1180--1189.

\bibitem{ganin2016domain}
Y.~Ganin, E.~Ustinova, H.~Ajakan, P.~Germain, H.~Larochelle, F.~Laviolette,
  M.~Marchand, and V.~Lempitsky, ``Domain-adversarial training of neural
  networks,'' \emph{JMLR}, vol.~17, no.~59, pp. 1--35, 2016.

\bibitem{tzeng2017adversarial}
E.~Tzeng, J.~Hoffman, K.~Saenko, and T.~Darrell, ``Adversarial discriminative
  domain adaptation,'' in \emph{CVPR}, 2017, pp. 2962--2971.

\bibitem{zeng2014deep}
X.~Zeng, W.~Ouyang, M.~Wang, and X.~Wang, ``Deep learning of scene-specific
  classifier for pedestrian detection,'' in \emph{ECCV}.\hskip 1em plus 0.5em
  minus 0.4em\relax Springer, 2014, pp. 472--487.

\bibitem{sun2016deep}
B.~Sun and K.~Saenko, ``Deep coral: Correlation alignment for deep domain
  adaptation,'' in \emph{ECCV}, 2016, pp. 443--450.

\bibitem{bousmalis2016domain}
K.~Bousmalis, G.~Trigeorgis, N.~Silberman, D.~Krishnan, and D.~Erhan, ``Domain
  separation networks,'' in \emph{NeurIPS}, 2016, pp. 343--351.

\bibitem{liu2016coupled}
M.-Y. Liu and O.~Tuzel, ``Coupled generative adversarial networks,'' in
  \emph{NeurIPS}, 2016, pp. 469--477.

\bibitem{kang2018deep}
G.~Kang, L.~Zheng, Y.~Yan, and Y.~Yang, ``Deep adversarial attention alignment
  for unsupervised domain adaptation: the benefit of target expectation
  maximization,'' in \emph{ECCV}, 2018.

\bibitem{long2018conditional}
M.~Long, Z.~Cao, J.~Wang, and M.~I. Jordan, ``Conditional adversarial domain
  adaptation,'' in \emph{NeurIPS}, 2018, pp. 1647--1657.

\bibitem{gretton2012kernel}
A.~Gretton, K.~M. Borgwardt, M.~J. Rasch, B.~Sch{\"o}lkopf, and A.~Smola, ``A
  kernel two-sample test,'' \emph{JMLR}, vol.~13, no. Mar, pp. 723--773, 2012.

\bibitem{zhang2019recent}
J.~Zhang, W.~Li, P.~Ogunbona, and D.~Xu, ``Recent advances in transfer learning
  for cross-dataset visual recognition: A problem-oriented perspective,''
  \emph{ACM Computing Surveys (CSUR)}, vol.~52, no.~1, p.~7, 2019.

\bibitem{kumar2010self}
M.~P. Kumar, B.~Packer, and D.~Koller, ``Self-paced learning for latent
  variable models,'' in \emph{NeurIPS}, 2010, pp. 1189--1197.

\bibitem{wang2016temporal}
L.~Wang, Y.~Xiong, Z.~Wang, Y.~Qiao, D.~Lin, X.~Tang, and L.~Van~Gool,
  ``Temporal segment networks: Towards good practices for deep action
  recognition,'' in \emph{ECCV}.\hskip 1em plus 0.5em minus 0.4em\relax
  Springer, 2016, pp. 20--36.

\bibitem{zhang2018collaborative}
W.~Zhang, W.~Ouyang, W.~Li, and D.~Xu, ``Collaborative and adversarial network
  for unsupervised domain adaptation,'' in \emph{CVPR}, 2018, pp. 3801--3809.

\bibitem{kulis2011}
B.~Kulis, K.~Saenko, and T.~Darrell, ``What you saw is not what you get: Domain
  adaptation using asymmetric kernel transforms,'' in \emph{CVPR}, 2011, pp.
  1785--1792.

\bibitem{gopalan2011domain}
R.~Gopalan, R.~Li, and R.~Chellappa, ``Domain adaptation for object
  recognition: An unsupervised approach,'' in \emph{ICCV}.\hskip 1em plus 0.5em
  minus 0.4em\relax IEEE, 2011, pp. 999--1006.

\bibitem{gong2012geodesic}
B.~Gong, Y.~Shi, F.~Sha, and K.~Grauman, ``Geodesic flow kernel for
  unsupervised domain adaptation,'' in \emph{CVPR}.\hskip 1em plus 0.5em minus
  0.4em\relax IEEE, 2012, pp. 2066--2073.

\bibitem{sugiyama2008direct}
M.~Sugiyama, S.~Nakajima, H.~Kashima, P.~V. Buenau, and M.~Kawanabe, ``Direct
  importance estimation with model selection and its application to covariate
  shift adaptation,'' in \emph{NeurIPS}, 2008, pp. 1433--1440.

\bibitem{huang2007correcting}
J.~Huang, A.~Gretton, K.~M. Borgwardt, B.~Sch{\"o}lkopf, and A.~J. Smola,
  ``Correcting sample selection bias by unlabeled data,'' in \emph{NeurIPS},
  2007, pp. 601--608.

\bibitem{bruzzone2010domain}
L.~Bruzzone and M.~Marconcini, ``Domain adaptation problems: A {DASVM}
  classification technique and a circular validation strategy,'' \emph{T-PAMI},
  vol.~32, no.~5, pp. 770--787, 2010.

\bibitem{bickel2007discriminative}
S.~Bickel, M.~Br{\"u}ckner, and T.~Scheffer, ``Discriminative learning for
  differing training and test distributions,'' in \emph{ICML}.\hskip 1em plus
  0.5em minus 0.4em\relax ACM, 2007, pp. 81--88.

\bibitem{li2018domain}
W.~Li, Z.~Xu, D.~Xu, D.~Dai, and L.~Van~Gool, ``Domain generalization and
  adaptation using low rank exemplar {SVM}s,'' \emph{T-PAMI}, vol.~40, no.~5,
  pp. 1114--1127, 2018.

\bibitem{mancini2018boosting}
M.~Mancini, L.~Porzi, S.~R. Bul{\`o}, B.~Caputo, and E.~Ricci, ``Boosting
  domain adaptation by discovering latent domains,'' in \emph{CVPR}, 2018.

\bibitem{carlucci2017autodial}
F.~M. Carlucci, L.~Porzi, B.~Caputo, E.~Ricci, and S.~R. Bul{\`o}, ``Autodial:
  Automatic domain alignment layers.'' in \emph{ICCV}, 2017, pp. 5077--5085.

\bibitem{pinheiro2018unsupervised}
P.~O. Pinheiro and A.~Element, ``Unsupervised domain adaptation with similarity
  learning,'' in \emph{CVPR}, 2018.

\bibitem{tzeng2015simultaneous}
E.~Tzeng, J.~Hoffman, T.~Darrell, and K.~Saenko, ``Simultaneous deep transfer
  across domains and tasks,'' in \emph{ICCV}.\hskip 1em plus 0.5em minus
  0.4em\relax IEEE, 2015, pp. 4068--4076.

\bibitem{sankaranarayanan2017generate}
S.~Sankaranarayanan, Y.~Balaji, C.~D. Castillo, and R.~Chellappa, ``Generate to
  adapt: Aligning domains using generative adversarial networks,'' in
  \emph{CVPR}, 2018.

\bibitem{hu2018duplex}
L.~Hu, M.~Kan, S.~Shan, and X.~Chen, ``Duplex generative adversarial network
  for unsupervised domain adaptation,'' in \emph{CVPR}, 2018, pp. 1498--1507.

\bibitem{russo2017source}
P.~Russo, F.~M. Carlucci, T.~Tommasi, and B.~Caputo, ``From source to target
  and back: symmetric bi-directional adaptive gan,'' in \emph{CVPR}, 2018.

\bibitem{hoffman2017cycada}
J.~Hoffman, E.~Tzeng, T.~Park, J.-Y. Zhu, P.~Isola, K.~Saenko, A.~A. Efros, and
  T.~Darrell, ``Cycada: Cycle-consistent adversarial domain adaptation,'' in
  \emph{ICML}, 2018.

\bibitem{lee2018diverse}
H.-Y. Lee, H.-Y. Tseng, J.-B. Huang, M.~K. Singh, and M.-H. Yang, ``Diverse
  image-to-image translation via disentangled representations,'' in
  \emph{ECCV}, 2018.

\bibitem{saito2018maximum}
K.~Saito, K.~Watanabe, Y.~Ushiku, and T.~Harada, ``Maximum classifier
  discrepancy for unsupervised domain adaptation,'' in \emph{CVPR}, 2018.

\bibitem{rozantsev2018residual}
A.~Rozantsev, M.~Salzmann, and P.~Fua, ``Residual parameter transfer for deep
  domain adaptation,'' in \emph{CVPR}, 2018, pp. 4339--4348.

\bibitem{rozantsev2018beyond}
------, ``Beyond sharing weights for deep domain adaptation,'' \emph{T-PAMI},
  2018.

\bibitem{goodfellow2014generative}
I.~Goodfellow, J.~Pouget-Abadie, M.~Mirza, B.~Xu, D.~Warde-Farley, S.~Ozair,
  A.~Courville, and Y.~Bengio, ``Generative adversarial nets,'' in
  \emph{NeurIPS}, 2014, pp. 2672--2680.

\bibitem{ben2010theory}
S.~Ben-David, J.~Blitzer, K.~Crammer, A.~Kulesza, F.~Pereira, and J.~W.
  Vaughan, ``A theory of learning from different domains,'' \emph{Machine
  learning}, vol.~79, no.~1, pp. 151--175, 2010.

\bibitem{gulrajani2017improved}
I.~Gulrajani, F.~Ahmed, M.~Arjovsky, V.~Dumoulin, and A.~C. Courville,
  ``Improved training of wasserstein gans,'' in \emph{NeurIPS}, 2017, pp.
  5767--5777.

\bibitem{chen2011co}
M.~Chen, K.~Q. Weinberger, and J.~Blitzer, ``Co-training for domain
  adaptation,'' in \emph{NeurIPS}, 2011, pp. 2456--2464.

\bibitem{blum1998combining}
A.~Blum and T.~Mitchell, ``Combining labeled and unlabeled data with
  co-training,'' in \emph{Proceedings of the eleventh annual conference on
  Computational learning theory}.\hskip 1em plus 0.5em minus 0.4em\relax ACM,
  1998, pp. 92--100.

\bibitem{saito2017asymmetric}
K.~Saito, Y.~Ushiku, and T.~Harada, ``Asymmetric tri-training for unsupervised
  domain adaptation,'' in \emph{ICML}.\hskip 1em plus 0.5em minus 0.4em\relax
  JMLR.org, 2017, pp. 2988--2997.

\bibitem{bengio2009curriculum}
Y.~Bengio, J.~Louradour, R.~Collobert, and J.~Weston, ``Curriculum learning,''
  in \emph{ICML}.\hskip 1em plus 0.5em minus 0.4em\relax ACM, 2009, pp. 41--48.

\bibitem{tang2012shifting}
K.~Tang, V.~Ramanathan, L.~Fei-Fei, and D.~Koller, ``Shifting weights: Adapting
  object detectors from image to video,'' in \emph{NeurIPS}, 2012, pp.
  638--646.

\bibitem{zhang2017curriculum}
Y.~Zhang, P.~David, and B.~Gong, ``Curriculum domain adaptation for semantic
  segmentation of urban scenes,'' in \emph{ICCV}, 2017, pp. 2020--2030.

\bibitem{zou2018unsupervised}
Y.~Zou, Z.~Yu, B.~Vijaya~Kumar, and J.~Wang, ``Unsupervised domain adaptation
  for semantic segmentation via class-balanced self-training,'' in \emph{ECCV},
  2018, pp. 289--305.

\bibitem{zhang2011multi}
D.~Zhang, J.~He, Y.~Liu, L.~Si, and R.~Lawrence, ``Multi-view transfer learning
  with a large margin approach,'' in \emph{SIGKDD}.\hskip 1em plus 0.5em minus
  0.4em\relax ACM, 2011, pp. 1208--1216.

\bibitem{yang2013multi}
P.~Yang and W.~Gao, ``Multi-view discriminant transfer learning,'' in
  \emph{IJCAI}, 2013.

\bibitem{niu2015multi}
L.~Niu, W.~Li, and D.~Xu, ``Multi-view domain generalization for visual
  recognition,'' in \emph{ICCV}, 2015, pp. 4193--4201.

\bibitem{ding2018robust}
Z.~Ding, M.~Shao, and Y.~Fu, ``Robust multi-view representation: a unified
  perspective from multi-view learning to domain adaption,'' in
  \emph{IJCAI}.\hskip 1em plus 0.5em minus 0.4em\relax AAAI Press, 2018, pp.
  5434--5440.

\bibitem{wang2013action}
H.~Wang and C.~Schmid, ``Action recognition with improved trajectories,'' in
  \emph{ICCV}, 2013, pp. 3551--3558.

\bibitem{ji20133d}
S.~Ji, W.~Xu, M.~Yang, and K.~Yu, ``{3D} convolutional neural networks for
  human action recognition,'' \emph{T-PAMI}, vol.~35, no.~1, pp. 221--231,
  2013.

\bibitem{karpathy2014large}
A.~Karpathy, G.~Toderici, S.~Shetty, T.~Leung, R.~Sukthankar, and L.~Fei-Fei,
  ``Large-scale video classification with convolutional neural networks,'' in
  \emph{CVPR}, 2014, pp. 1725--1732.

\bibitem{yue2015beyond}
J.~Yue-Hei~Ng, M.~Hausknecht, S.~Vijayanarasimhan, O.~Vinyals, R.~Monga, and
  G.~Toderici, ``Beyond short snippets: Deep networks for video
  classification,'' in \emph{CVPR}, 2015, pp. 4694--4702.

\bibitem{wang2015action}
L.~Wang, Y.~Qiao, and X.~Tang, ``Action recognition with trajectory-pooled
  deep-convolutional descriptors,'' in \emph{CVPR}, 2015, pp. 4305--4314.

\bibitem{tran2015learning}
D.~Tran, L.~Bourdev, R.~Fergus, L.~Torresani, and M.~Paluri, ``Learning
  spatiotemporal features with 3d convolutional networks,'' in \emph{ICCV},
  2015, pp. 4489--4497.

\bibitem{feichtenhofer2016convolutional}
C.~Feichtenhofer, A.~Pinz, and A.~Zisserman, ``Convolutional two-stream network
  fusion for video action recognition,'' in \emph{CVPR}, 2016, pp. 1933--1941.

\bibitem{niu2016exploiting}
L.~Niu, W.~Li, and D.~Xu, ``Exploiting privileged information from web data for
  action and event recognition,'' \emph{IJCV}, vol. 118, no.~2, pp. 130--150,
  2016.

\bibitem{yu2018exploiting}
F.~Yu, X.~Wu, Y.~Sun, and L.~Duan, ``Exploiting images for video recognition
  with hierarchical generative adversarial networks,'' in \emph{IJCAI}, 2018.

\bibitem{Jamal2018ddaa}
A.~Jamal, P.~V. Namboodiri, D.~Deodhare1, and K.~Venkatesh, ``Deep domain
  adaptation in action space,'' in \emph{BMVC}, 2018.

\bibitem{busto2018open}
P.~P. Busto, A.~Iqbal, and J.~Gall, ``Open set domain adaptation for image and
  action recognition,'' \emph{T-PAMI}, 2018.

\bibitem{ben2007analysis}
S.~Ben-David, J.~Blitzer, K.~Crammer, and F.~Pereira, ``Analysis of
  representations for domain adaptation,'' in \emph{Advances in neural
  information processing systems}, 2007, pp. 137--144.

\bibitem{bazaraa2013nonlinear}
M.~S. Bazaraa, H.~D. Sherali, and C.~M. Shetty, \emph{Nonlinear programming:
  theory and algorithms}.\hskip 1em plus 0.5em minus 0.4em\relax John Wiley \&
  Sons, 2013.

\bibitem{saenko2010adapting}
K.~Saenko, B.~Kulis, M.~Fritz, and T.~Darrell, ``Adapting visual category
  models to new domains,'' in \emph{ECCV}, 2010, pp. 213--226.

\bibitem{peng2018visda}
X.~Peng, B.~Usman, N.~Kaushik, D.~Wang, J.~Hoffman, K.~Saenko, X.~Roynard,
  J.-E. Deschaud, F.~Goulette, T.~L. Hayes \emph{et~al.}, ``Visda: A
  synthetic-to-real benchmark for visual domain adaptation,'' in \emph{CVPR-W},
  2018, pp. 2021--2026.

\bibitem{lin2014microsoft}
T.-Y. Lin, M.~Maire, S.~Belongie, J.~Hays, P.~Perona, D.~Ramanan,
  P.~Doll{\'a}r, and C.~L. Zitnick, ``Microsoft coco: Common objects in
  context,'' in \emph{ECCV}.\hskip 1em plus 0.5em minus 0.4em\relax Springer,
  2014, pp. 740--755.

\bibitem{pei2018multi}
Z.~Pei, Z.~Cao, M.~Long, and J.~Wang, ``Multi-adversarial domain adaptation,''
  in \emph{AAAI}, 2018.

\bibitem{saito2017adversarial}
K.~Saito, Y.~Ushiku, T.~Harada, and K.~Saenko, ``Adversarial dropout
  regularization,'' in \emph{ICLR}, 2018.

\bibitem{Simonyan14c}
K.~Simonyan and A.~Zisserman, ``Very deep convolutional networks for
  large-scale image recognition,'' \emph{CoRR}, vol. abs/1409.1556, 2014.

\bibitem{huang2017densely}
G.~Huang, Z.~Liu, L.~van~der Maaten, and K.~Q. Weinberger, ``Densely connected
  convolutional networks,'' in \emph{CVPR}, 2017.

\bibitem{maaten2008visualizing}
L.~v.~d. Maaten and G.~Hinton, ``Visualizing data using t-sne,'' \emph{JMLR},
  vol.~9, no. Nov, pp. 2579--2605, 2008.

\bibitem{soomro2012ucf101}
K.~Soomro, A.~R. Zamir, and M.~Shah, ``{UCF101}: A dataset of 101 human actions
  classes from videos in the wild,'' \emph{arXiv preprint arXiv:1212.0402},
  2012.

\bibitem{kuehne2011hmdb}
H.~Kuehne, H.~Jhuang, E.~Garrote, T.~Poggio, and T.~Serre, ``{HMDB}: a large
  video database for human motion recognition,'' in \emph{ICCV}.\hskip 1em plus
  0.5em minus 0.4em\relax IEEE, 2011, pp. 2556--2563.

\bibitem{ioffe2015batch}
S.~Ioffe and C.~Szegedy, ``Batch normalization: accelerating deep network
  training by reducing internal covariate shift,'' in \emph{ICML}, 2015, pp.
  448--456.

\end{thebibliography}
}
\vspace{-10mm}
\begin{IEEEbiography}
[{\includegraphics[width=1in,height=1.25in,clip,keepaspectratio]{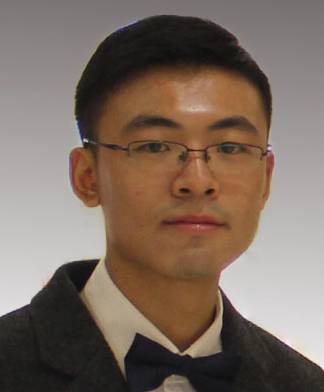}}]
{Weichen Zhang} received the BIT degree in School of Information Technology from the University of Sydney in 2017. He is currently working toward the PhD degree in the School of Electrical and Information Engineering, the University of Sydney. His current research interests include deep transfer learning and its applications in computer vision.
\end{IEEEbiography}

\vspace{-10mm}
\begin{IEEEbiography}
[{\includegraphics[width=1in,height=1.25in,clip,keepaspectratio]{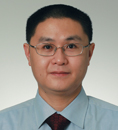}}]
{Dong Xu} received the BE and PhD degrees from University of Science and Technology of China, in 2001 and 2005, respectively. He was a post-doctoral research scientist with Columbia University, New York, NY, for one year, and also worked as a faculty member for more than eight years with Nanyang Technological University, Singapore. Currently, he is a professor (Chair in Computer Engineering) with the School of Electrical and Information Engineering, the University of Sydney, Australia. His current research interests include computer vision, statistical learning, and multimedia content analysis. He is a fellow of IEEE and IAPR.
\end{IEEEbiography}

\vspace{-10mm}
\begin{IEEEbiography}
[{\includegraphics[width=1in,height=1.25in,clip,keepaspectratio]{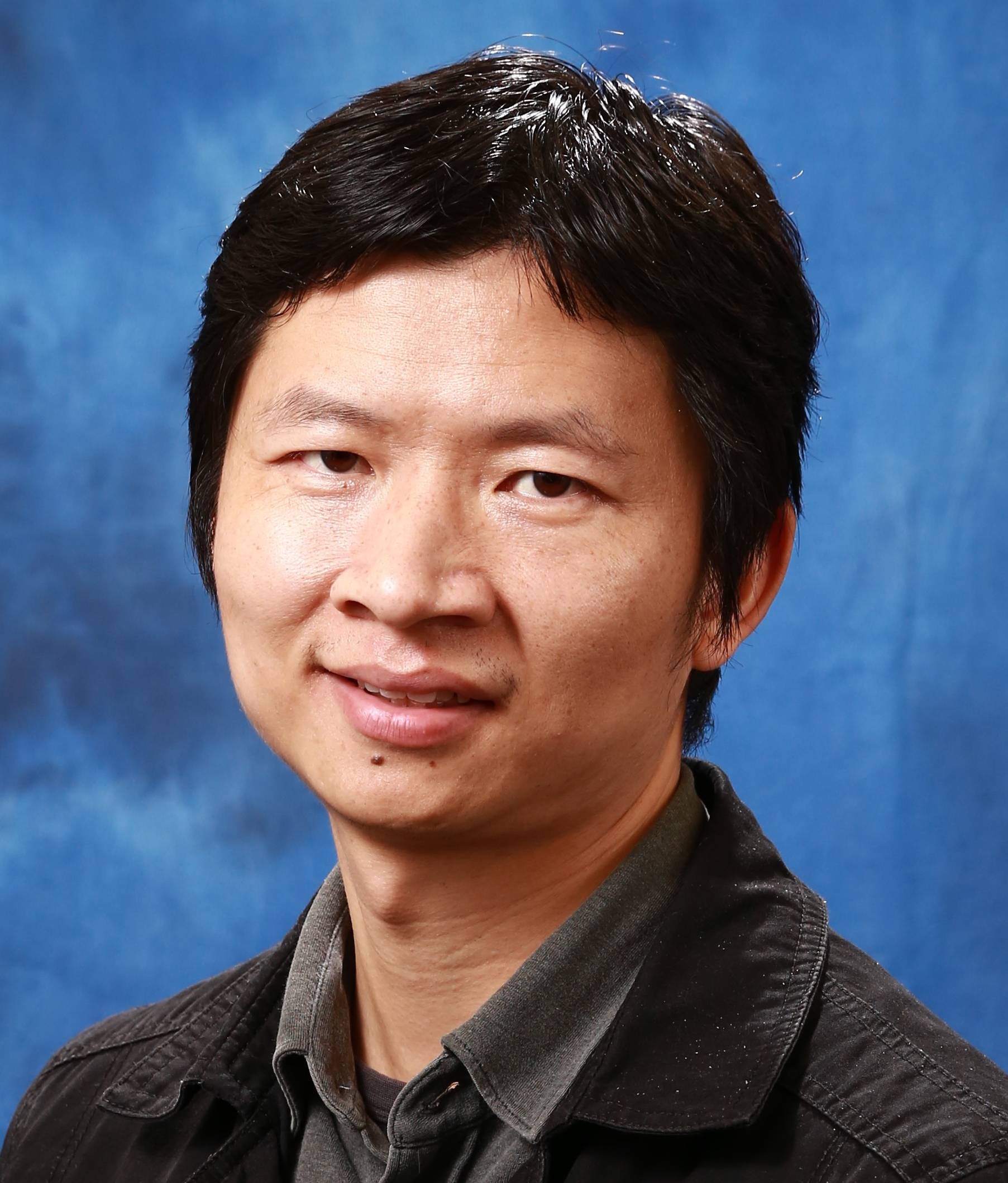}}]
{Wanli Ouyang} received the PhD degree in the Department of Electronic Engineering, Chinese University of Hong Kong. Since 2017, he is a senior lecturer with the University of Sydney. His research interests include image processing, computer vision, and pattern recognition. He is a senior member of the IEEE.
\end{IEEEbiography}

\vspace{-10mm}
\begin{IEEEbiography}[{\includegraphics[width=1in,height=1.25in,clip,keepaspectratio]{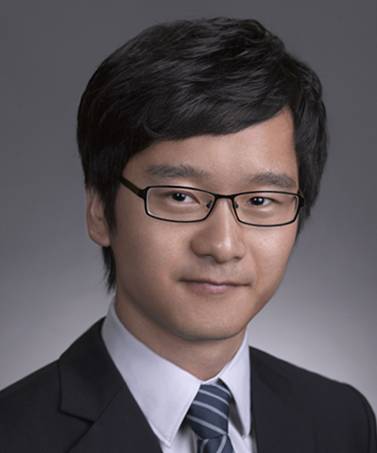}}]
{Wen Li} is a postdoctoral researcher with the Computer Vision Laboratory, ETH Zurich, Switzerland. He received the Ph.D. degree from the Nanyang Technological University, Singapore in 2015. Before that, he received the B.S. and M.Eng degrees from the Beijing Normal University, Beijing, China, in 2007 and 2010, respectively. His main interests include transfer learning, multi-view learning, multiple kernel learning, and their applications in computer vision.
\end{IEEEbiography}

% if you will not have a photo at all:
% \begin{IEEEbiographynophoto}{John Doe}
% Biography text here.
% \end{IEEEbiographynophoto}

% insert where needed to balance the two columns on the last page with
% biographies
%\newpage

% \begin{IEEEbiographynophoto}{Jane Doe}
% Biography text here.
% \end{IEEEbiographynophoto}

% You can push biographies down or up by placing
% a \vfill before or after them. The appropriate
% use of \vfill depends on what kind of text is
% on the last page and whether or not the columns
% are being equalized.

%\vfill

% Can be used to pull up biographies so that the bottom of the last one
% is flush with the other column.
%\enlargethispage{-5in}

% that's all folks
\end{document}